\documentclass{article}

\PassOptionsToPackage{numbers,sort&compress}{natbib}
\usepackage[utf8]{inputenc} 
\usepackage[T1]{fontenc}    
\usepackage{amsmath}
\usepackage{amssymb}
\usepackage{amsthm}

\usepackage[preprint]{neurips_2026}

\usepackage{hyperref}       
\usepackage{url}            
\usepackage{booktabs}       
\usepackage{amsfonts}       
\usepackage{nicefrac}       
\usepackage{microtype}      
\usepackage{xcolor}         
\usepackage{float}

\usepackage[linesnumbered,ruled]{algorithm2e}
\usepackage{amsopn}
\usepackage{bm}
\usepackage{bbm}
\usepackage{braket}
\usepackage{comment}
\usepackage{extarrows}
\usepackage{graphicx}
\usepackage{mathrsfs}  
\usepackage[mathscr]{euscript}
\usepackage{subcaption}
\usepackage[capitalise]{cleveref}
\usepackage{wrapfig}
\usepackage{adjustbox}
\usepackage{enumitem}

\makeatletter
\DeclareSymbolFont{extraup}{U}{zavm}{m}{n}
\DeclareMathSymbol{\newcheckmark}{\mathalpha}{extraup}{128}
\DeclareMathSymbol{\xmark}{\mathalpha}{extraup}{129}
\makeatother

\hypersetup{
    colorlinks,
    citecolor=blue,
    filecolor=black,
    linkcolor=blue,
    urlcolor=black
}

\definecolor{somecolor}{RGB}{204, 88, 3}

\def\h{\mathbf{h}}

\def\x{\mathbf{x}}

\def\W{\mathbf{W}}

\DeclareMathOperator{\id}{id}
\DeclareMathOperator{\tr}{tr}

\newtheorem*{theorem*}{Theorem}
\newtheorem{theorem}{Theorem}[section]
\newtheorem{proposition}[theorem]{Proposition}
\newtheorem{definition}[theorem]{Definition}
\newtheorem{lemma}[theorem]{Lemma}
\newtheorem{corollary}[theorem]{Corollary}

\newtheoremstyle{boldremark} 
{6pt}   
{6pt}   
{\normalfont} 
{}      
{\bfseries} 
{.}     
{.5em}  
{\thmname{#1}\thmnumber{ #2}\thmnote{\textbf{ (#3)}}} 

\theoremstyle{boldremark}
\newtheorem{remark}{Remark}

\title{How Long Does Infinite Width Last?\\
Signal Propagation in Long-Range Linear Recurrences}

\author{%
  Mariia Seleznova\\
Technische Universität Darmstadt\\
Max Planck Institute of Quantum Optics
}

\begin{document}

\maketitle

\begin{abstract}
We study signal propagation in linear recurrent models at finite width. While existing signal propagation theory relies predominantly on the infinite-width limit, it remains unclear for how long that approximation remains accurate when recurrent depth $t$ grows jointly with width $n$. This question is especially relevant for modern recurrent sequence models, whose natural operating regime involves long input sequences, i.e., large $t$. We derive exact finite-width formulas for the hidden state signal energies in linear recurrences under complex Gaussian initialization. Using these formulas, we identify the joint depth--width scaling regimes that govern signal propagation: (i) a \emph{subcritical regime} $t=o(\sqrt n)$, in which the infinite-width approximation remains valid; (ii) a \emph{critical regime} $t\sim c\sqrt n$, in which non-negligible deviations from infinite-width predictions appear and a nontrivial joint scaling limit emerges; and (iii) a \emph{supercritical regime} $t\gg \sqrt n$, in which finite-width effects dominate. Thus, our results pinpoint the precise recurrent depth scale at which infinite-width theory breaks down in long-range linear recurrences. In turn, this shows when standard initialization schemes, such as Glorot, become unstable. More broadly, our results show that what constitutes "large width" depends crucially on the architecture: recurrent weight sharing makes finite-width effects visible at substantially shorter depths than in feedforward networks.
\end{abstract}

\section{Introduction}\label{section:intro}

Understanding how signals propagate through the layers of a neural network (NN) is a foundational problem in deep learning (DL) theory. When signals are repeatedly transformed by many layers of a deep network, their magnitudes may progressively grow or decay, eventually leading to numerical failure. This phenomenon is widely known as the \emph{vanishing/exploding signals problem}, and it poses a fundamental obstacle to training deep NNs~\citep{bengio1994learning,pascanu2013difficulty}. Ensuring stable signal propagation has therefore long been a central design principle in DL methods, motivating widely used initialization schemes~\citep{glorot2010understanding,he2015delving,schoenholz2016deep} and architectural choices~\citep{ba2016layer,he2016deep,hochreiter1997long}.

Despite the foundational nature and long history of this question, the theoretical understanding of signal propagation in DL models remains limited. Much of the existing literature relies on the \emph{infinite-width limit}~\citep{schoenholz2016deep,poole2016exponential,yang2019scaling,yang2021tensor}, where the network's width $n$ (i.e, the number of neurons per layer) tends to infinity. In this regime, mean-field-type approximations yield tractable descriptions of signal statistics and lead to elegant signal propagation theories for many DL architectures~\citep{yang2019scaling,yang2021tensor}, including recurrent models~\citep{chen2018dynamical,gilboa2019dynamical,alemohammad2021recurrent}. Realistic networks, however, operate at finite width. In this setting, even small deviations from the infinite-width limit can accumulate over many layers or recurrent steps, eventually rendering infinite-width predictions inaccurate in sufficiently deep architectures or long-range recurrences~\citep{roberts2022principles,hanin2018neural,hanin2019finite,seleznova2022neural,bar2025revisiting}. This leads to the central question of this paper:
\begin{center}
	\emph{How long does the infinite-width approximation remain accurate in recurrent models?}
\end{center}

For fully-connected feedforward networks at random initialization, the answer is relatively well understood. The deviations from infinite-width behavior become significant when depth scales proportionally with width, i.e., $L\sim cn$, where $L$ is the number of layers and $n$ is the layer width~\citep{hanin2018neural,hanin2019finite,seleznova2022neural,hanin2024random}. That is, for fully connected feedforward networks, the critical depth scale at which finite-width effects begin to qualitatively alter signal propagation is by now known.

By contrast, much less is known for recurrent networks. Although existing work suggests that finite-width corrections behave differently in recurrent and feedforward architectures~\citep{segadlo2022unified,grosvenor2022edge}, their precise effect for long-range recurrent signal propagation remains unclear. Intuitively, finite-width effects accumulate faster in recurrences, since the same weight matrix is applied repeatedly and its dominant spectral directions are therefore reinforced across recurrent steps $t$. \citet{bar2025revisiting} take a first step toward formalizing this intuition in linear recurrent sequence models with Gaussian initialization, showing that accumulated finite-width effects cause signal explosion when $t\gg \sqrt n$. However, their argument gives only a crude lower bound on the signal size. Thus, even for linear recurrences at random initialization, the sharp depth scale at which infinite-width predictions break down has remained unknown.

This gap is especially striking from the perspective of modern long-range sequence modeling, where architectures built around \emph{linear recurrences}---such as state-space models (SSMs)~\citep{gu2021efficiently,gu2020hippo,de2024griffin} and linear recurrent units~(LRUs)~\citep{orvieto2023resurrecting}---are actively explored as computationally efficient alternatives to transformers for long-context tasks. These models have recently achieved competitive empirical performance while retaining the computational benefits of linear recurrence. Since they are designed for long-range tasks, their natural operating regime is that of long input sequences, or equivalently large recurrent depth $t$. Yet this is precisely the regime in which signal propagation theory remains weak.

This paper helps close this gap by characterizing finite-width signal propagation in randomly-initialized linear recurrences and identifying the recurrent depth scaling regimes in which the infinite-width approximation remains accurate and those in which it fails. We first specify the considered setting in Section~\ref{section:setting}, and then summarize the main contributions in Section~\ref{section:contributions}.

\subsection{Problem Setting }\label{section:setting}

Our main object of interest is the {linear recurrent unit} (LRU), which has recently re-emerged as a building block for long-range sequence models~\citep{orvieto2023resurrecting}. On the way, our analysis also yields results for the linear recurrent network without sequential inputs. Thus, the models~are:
\begin{align}
\text{\textbf{Linear recurrent unit (LRU):}}
\qquad
\h_{\textup{LRU}}^{(t)}
&:= \W \h_{\textup{LRU}}^{(t-1)} + \x^{(t)}
 = \sum_{k=0}^{t} \W^k \x^{(t-k)}, \ t \geq 1,
\label{eq:lru}\\
\text{\textbf{Linear recurrent network (RNN):}}
\qquad
\h_{\textup{RNN}}^{(t)}
&:= \W \h_{\textup{RNN}}^{(t-1)}
 = \W^t \x^{(0)}, \ t\geq 1.
\label{eq:linear-rnn}
\end{align}
Here $\h^{(t)}_{\cdot}\in\mathbb{C}^n$ is the hidden state vector at step $t\in\mathbb{N}$, $(\x^{(0)}, \dots, \x^{(t)})$ is the input sequence, where $\x^{(i)}\in\mathbb{R}^n$ for all $0\leq i \leq t$, and $\W\in\mathbb{C}^{n\times n}$ is the weight matrix. We set $\h_{\cdot}^{(0)}=\x^{(0)}$ in both models. The weight matrix $\W$ is initialized with i.i.d.\ proper complex Gaussian entries of variance~$1/n$, namely
\begin{equation}\label{eq:glorot}
    \W_{ij}\sim \mathcal N_{\mathbb{C}}(0,1/n), \quad \text{i.i.d. across } i,j=1, \dots, n.
\end{equation}

The main results in this work are formulated for complex weights. This choice is natural for two reasons. First, complex-valued weights are standard in modern linear recurrent models~\citep{orvieto2023resurrecting,gu2021efficiently}, since complex eigenvalues provide a convenient representation of oscillatory dynamics. Second, the complex Gaussian ensemble admits a particularly clean analysis, which allows us to derive sharp finite-width results for signal propagation. At the same time, we show that our main conclusions on the limitations of infinite-width theory remain valid for real weights (see Remark~\ref{remark:real_case}).

Our main quantitative objects are the normalized expected squared norms of the hidden states $\h_{\mathrm{LRU}}^{(t)}$ and $\h_{\mathrm{RNN}}^{(k)}$, which we call \emph{signal energies}. They are defined as follows:
\begin{equation}\label{eq:main_quantities}
   \textbf{Signal energies}: \quad S_{n,t}
   := \frac{1}{n}\mathbb E\bigl[\|\h_{\mathrm{LRU}}^{(t)}\|_2^2\bigr],
   \qquad
   Q_{n,k}
   := \frac{1}{n}\mathbb E\bigl[\|\h_{\mathrm{RNN}}^{(k)}\|_2^2\bigr],
\end{equation}
where the expectation is taken jointly over the weights $\W$ and the inputs $(\x^{(0)}, \dots, \x^{(t)})$. The same quantities are studied in classical works on signal propagation at infinite width~\citep{poole2016exponential,schoenholz2016deep}. Notably, standard Gaussian initialization schemes---such as He~\citep{he2015delving} and Glorot~\citep{glorot2010understanding}---are designed to stabilize these signal energies across layers at infinite width, so that $Q_{\infty,k} \approx Q_{\infty, k+1}$. However, in this work we explicitly study the behavior of these quantities at \emph{finite width}.

\subsection{Contributions}\label{section:contributions}

In the setting described above, the main contributions of this work are as follows:
\begin{itemize}
    \item \textbf{Exact finite-width characterization of signal propagation.} We derive exact closed-form expressions for signal energies in finite-width linear recurrences with complex Gaussian initialization. Specifically, we obtain expressions for $S_{n,t}$ and $Q_{n,k}$ (defined in \eqref{eq:main_quantities}) for finite $n,t$ and $k$. This result is given in Theorem~\ref{thm:main}. To the best of our knowledge, this is the first exact finite-width result for signal propagation in recurrent models, which provides a clear departure from the standard infinite-width signal propagation theory.
    
   \item \textbf{Sharp depth--width scaling laws.} By analyzing the finite-width expressions, we identify the recurrent depth--width scaling regimes that govern signal propagation. In particular, we characterize three regimes: (i) the \emph{subcritical regime} $k,t=o(\sqrt n)$, in which the infinite-width approximation remains accurate; (ii) the \emph{critical regime} $k,t\sim c\sqrt n$, in which a nontrivial joint depth--width limit emerges; and (iii) the \emph{supercritical regime} $k,t\gg \sqrt n$, in which finite-width effects dominate. The behavior of signal energies $S_{n,t}$ and $Q_{n,k}$ in these regimes is summarized in Figure~\ref{fig:scaling_laws}. Together, these results yield a precise joint scaling signal propagation theory for recurrent depth and width in linear recurrences.
 
    \item \textbf{Limitations of infinite-width theory for long-range recurrent models.} The identified critical depth scale $t\sim c\sqrt n$, at which finite-width corrections qualitatively alter signal propagation, makes explicit when infinite-width theory breaks down in long-range recurrent models. This has two main consequences. First, standard Glorot initialization~\citep{glorot2010understanding}, while stable at infinite width, no longer ensures stable propagation from the critical regime onward. Second, recurrent and feedforward architectures behave qualitatively differently at finite width, despite being indistinguishable at infinite width. Although our theory concerns linear recurrences, numerical experiments suggest that these limitations of infinite-width theory persist in certain nonlinear models (see Figure~\ref{fig:nonlinear}).

\end{itemize}
In addition, we validate our theoretical results with numerical experiments, which confirm the predicted signal propagation behavior.

\begin{figure}[h]
\centering
\includegraphics[width=0.45\linewidth]{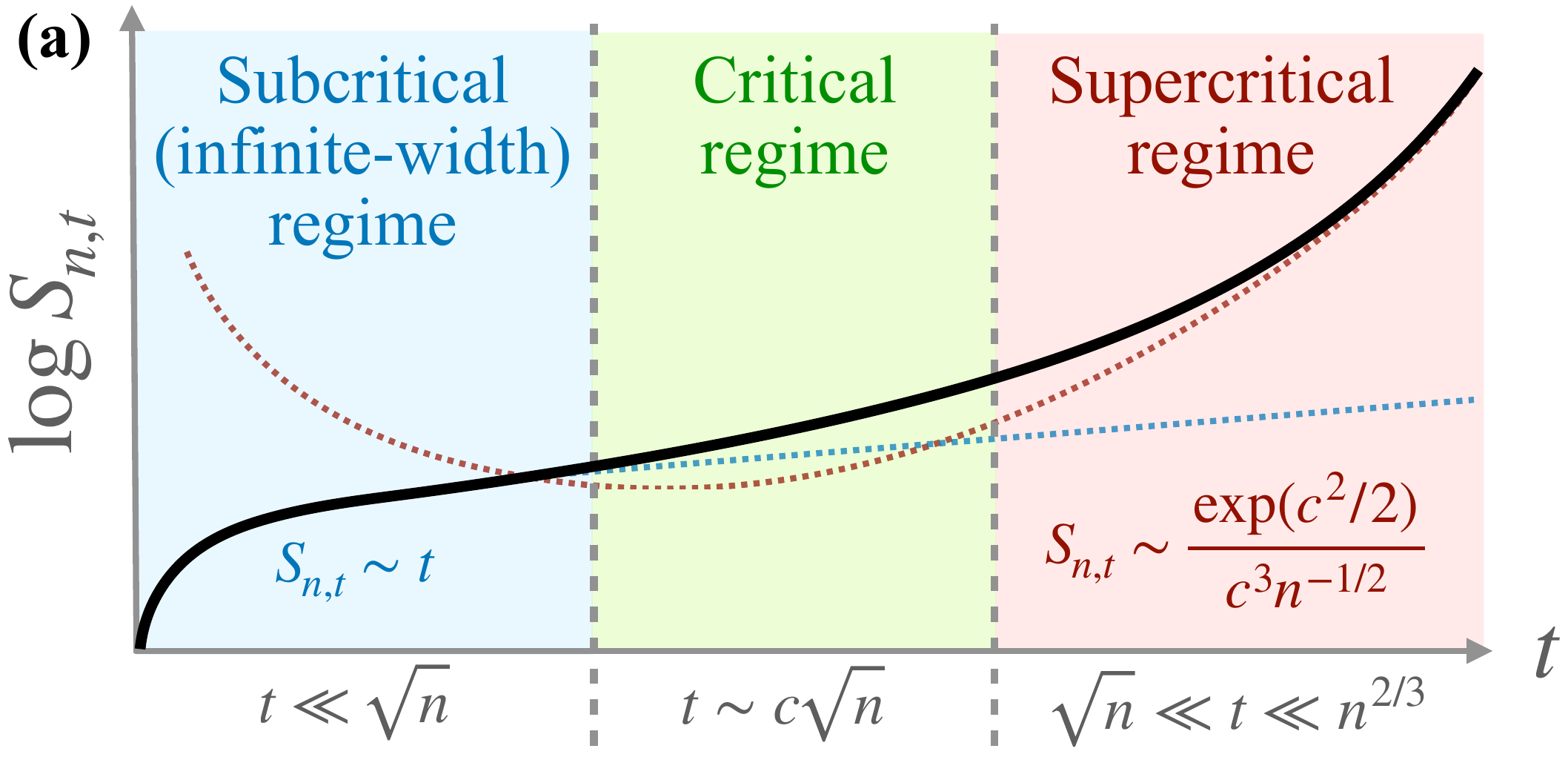}\hspace{6ex}
\includegraphics[width=0.45\linewidth]{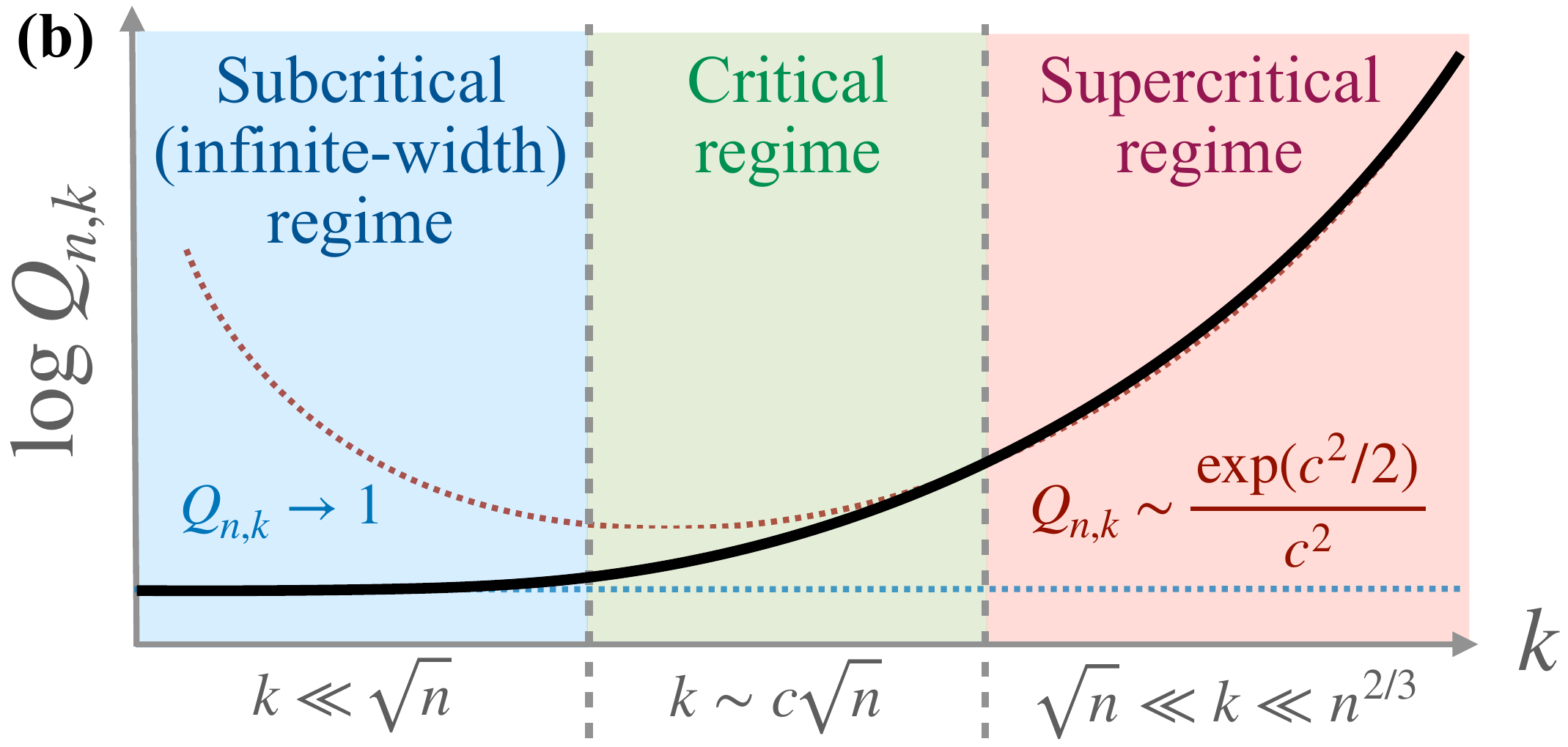}
\caption{\textbf{Depth--width scaling laws in linear recurrences.} The behavior of the signal energies under three joint scaling regimes of recurrent depth and width: (i) the \emph{subcritical regime} $k,t=o(\sqrt n)$, (ii) the \emph{critical regime} $k,t\sim c\sqrt n$, and (iii) the \emph{supercritical regime} subset $\sqrt n \ll k,t\ll n^{2/3}$. Panel~(a) shows the LRU signal energy $S_{n,t}$, and panel~(b) shows the RNN signal energy $Q_{n,k}$, as defined in~\eqref{eq:main_quantities}. Blue dotted lines indicate the continuation of the predicted infinite-width behavior beyond the subcritical regime, while red dotted lines indicate the continuation of the supercritical behavior beyond the displayed supercritical regime. Both quantities are shown on a logarithmic scale.}
\label{fig:scaling_laws}
\end{figure}

\paragraph{Outline.} Section~\ref{section:main} states the main finite-width signal propagation result. Section~\ref{section:scaling_regimes} analyzes the resulting depth--width scaling laws, and Section~\ref{section:numerics} presents numerical experiments. Section~\ref{section:discussion} concludes with a discussion. Additional related work is discussed in Appendix~\ref{section:related_works}.

\section{Recurrent Signal Propagation at Finite Width}\label{section:main}

We now state the main finite-width signal propagation theorem, proven in Appendix~\ref{section:proof_main}.

\begin{theorem}[Signal propagation at finite width]\label{thm:main}
	Suppose the weight matrix $\W\in\mathbb{C}^{n\times n}$ has i.i.d.\ proper complex Gaussian entries with variance $1/n$, as in~\eqref{eq:glorot}. Let $(\x^{(0)},\dots,\x^{(t)})$ be an input sequence independent of $\W$, and assume that $\mathbb E[\x^{(s)}\x^{(s)\top}] = \mathbb I_n$ for all $0\le s\le t$.
	No independence assumptions across time are required. Then the following identities hold for the hidden state signal energies of the LRU and the linear RNN, defined in~\eqref{eq:main_quantities}:
	\begin{equation}
		S_{n,t} = \sum_{k=0}^{t} Q_{n,k},
		\qquad
		Q_{n,k} = \frac{F^+_{k+2}(n)-F^-_{k+2}(n)}{n^{k+1}(k+1)(k+2)}.
	\end{equation}
	Here $F_p^+(n)$ and $F_p^-(n)$ denote the rising and falling factorials, respectively:
	\begin{equation}
		F_p^+(n):=n(n+1)\cdots(n+p-1),
		\qquad
		F_p^-(n):=n(n-1)\cdots(n-p+1).
	\end{equation}
\end{theorem}
This theorem gives exact closed-form expressions for the signal energies of linear recurrences at finite width. Its main value is that it enables a precise analysis of signal propagation under arbitrary joint scaling of recurrent depth and width, for example in regimes of the form $t,k \sim n^\alpha$. Section~\ref{section:scaling_regimes} develops the corresponding depth--width scaling laws, summarized in Figure~\ref{fig:scaling_laws}.

Below, we make several remarks on extensions and variants of this result.

\begin{remark}[General input distributions]\label{remark:general_covariances}
	Theorem~\ref{thm:main} is stated for whitened inputs for simplicity. However, the same argument extends immediately to arbitrary input covariances $\Sigma^{(0)},\dots,\Sigma^{(t)}$, where $\mathbb E[\x^{(s)}\x^{(s)\top}] = \Sigma^{(s)}$ for all $0\leq s\leq t$.
	In that case, the expected LRU signal energy is
	\begin{equation}
		S_{n,t}
		= \sum_{k=0}^t  \frac{\tr\!\bigl(\Sigma^{(t-k)}\bigr)}{n}\, \tilde Q_{n,k},
	\end{equation}
	where $\tilde Q_{n,k}$ denotes the RNN signal energy under whitened inputs, as given in Theorem~\ref{thm:main}. In fact, the proof in Appendix~\ref{section:proof_main} is formulated in this general setting. Since we impose no additional assumptions on cross-time correlations or the distribution of the input sequence $(\x^{(0)},\dots,\x^{(t)})$, the result is fully general with respect to the inputs, apart from their independence from the weights $\W$.
\end{remark}

\begin{remark}[Gradient propagation]\label{remark:gradients}
The quantities in Theorem~\ref{thm:main} also govern backward (gradient) propagation through the linear recurrence. Indeed, consider the output $y=\langle \mathbf c, \h_{\textup{ RNN}}^{(T)}\rangle$, where $\mathbf c$ are the output weights independent of $\W$. Backpropagation over $t$ steps in this model is controlled by
\begin{equation}\label{eq:gradient_recursion}
    \boldsymbol\delta^{(T-t)}:= \dfrac{\partial y}{\partial \h_{\textup{RNN}}^{(T-t)}}=(\W^*)^t\mathbf c,
    \qquad 0\leq t\leq T.
\end{equation}
This is the backward analogue of the forward recursion, and the same unitary invariance argument as in Lemma~\ref{lemma:Q_trace} gives:
\begin{equation}\label{eq:gradient_energy}
    \frac{1}{n}\mathbb E\!\left[\|\boldsymbol\delta^{(T-t)}\|_2^2\right]
    =\frac{\mathbb E\left[\|\mathbf c\|_2^2\right]}{n}\,Q_{n,t},
\end{equation}
where $Q_{n,t}$ is the signal energy for whitened inputs, as computed in Theorem~\ref{thm:main}. This shows that the gradient energy obeys the same scaling laws as the forward signal energy. 
\end{remark}

\begin{remark}[Real weights]\label{remark:real_case}
	While our main result is stated for complex weights, which are common in practical linear recurrences, it also directly yields a lower bound for the corresponding real Gaussian model. Let $Q_{n,k}^{\mathbb R}$ and $S_{n,t}^{\mathbb R}$ denote the analogues of $Q_{n,k}$ and $S_{n,t}$ for a real Gaussian weight matrix with i.i.d.\ entries of variance $1/n$. Then, with an additional assumption that cross-time input correlations satisfy $\tr(\mathbb{E}\![\x^{(s)}\x^{(r)\top}])\geq 0$ for all $0\leq s,r\leq t$, we have:
	\begin{equation}
		Q_{n,k}^{\mathbb R} \geq Q_{n,k},
		\qquad
		S_{n,t}^{\mathbb R} \geq S_{n,t},
	\end{equation}
	for all $n,k,t$. In particular, any depth--width scaling regime that leads to exploding signals in the complex model also leads to exploding signals in the real model. We prove this result in Appendix~\ref{section:real_case}.
\end{remark}

\begin{remark}[Nonlinear recurrences]\label{remark:nonlinear}
While Theorem~\ref{thm:main} concerns linear recurrences, the departure from infinite-width theory identified here does not appear to be restricted to the linear setting. Our empirical results for RNNs with ReLU and LeakyReLU activations show that the corresponding signal energy $Q^\phi_{n,t}$ grows rapidly as a function of $c=t/\sqrt n$, qualitatively resembling the behavior observed in the linear case (see Figure~\ref{fig:nonlinear}). Notably, the signal growth for these activations is \emph{faster} than in the corresponding linear networks. A theoretical characterization of these finite-width effects in nonlinear RNNs remains open.
\end{remark}

\subsection{Proof Idea}\label{section:proof_ides}
The proof combines standard random matrix techniques with a reduction to a cycle counting problem for certain random permutations. The starting point is the observation that the main difficulty in computing both signal energies, $Q_{n,k}$ and $S_{n,t}$, is the evaluation of the finite-width \emph{trace moment}
\begin{equation}
	\frac{1}{n}\,\mathbb E  \Bigl [ \tr\bigl((\W^k)^*\W^k\bigr)\Bigr].
\end{equation}
The reduction of $Q_{n,k}$ and $S_{n,t}$ to this quantity is carried out in Appendix~\ref{section:reduction_to_trace} using standard properties of Gaussian random matrices. Thus, the core of the proof is to evaluate this trace moment exactly in the given complex Gaussian setting.

To evaluate the trace moment, we use a Wick expansion argument: we first expand the trace into a sum of monomials of the form $	\overline{\W}_{a_0 a_1}\,\overline{\W}_{a_1 a_2}\cdots \overline{\W}_{a_{k-1} a_k}\,
\W_{a_0 b_1}\cdots \W_{b_{k-1} a_k}$, and then apply Wick's theorem for Gaussian moments~\citep{isserlis1918formula,wick1950evaluation} to compute their expectations. This is a standard way of converting random matrix moments into combinatorial counting problems, and we develop it in our setting in Appendix~\ref{section:trace_moment}. Related ideas have been used successfully in finite-width analyses of fully-connected feedforward networks~\citep{yaida2020non,roberts2022principles}. In the recurrent setting, however, the resulting combinatorics are substantially more intricate because the same random matrix is reused across time. The key simplification in our setting is that, for complex Gaussian recurrences, the Wick expansion takes a particularly clean form and can be indexed by permutations $\sigma\in S_k$. Each such permutation contributes according to the number of matrix indices that remain unconstrained after the corresponding contractions. Under the complex Gaussian initialization with variance $1/n$, this yields the exact representation
\begin{equation}
	\mathbb{E}\!\left[\frac{1}{n}\tr\bigl((\W^k)^*\W^k\bigr)\right]
	=
	\sum_{\sigma\in S_k} n^{-k -1 + F(\sigma)},
\end{equation}
where $F(\sigma)$ denotes the number of \emph{free indices}, i.e.\ the number of equivalence classes of indices $(a_0,\dots,a_k,b_1,\dots,b_{k-1})$ left unconstrained by the contractions associated with $\sigma$.

For a permutation $\sigma$, the quantity $F(\sigma)$ turns out to be exactly the number of cycles of the commutator $[\tau^{-1},\tilde\sigma^{-1}]$, where $\tau=(0,1,\dots,k)$ is a long cycle and $\tilde\sigma$ is the extension of $\sigma$ fixing $0$. This provides the link between the Wick expansion and permutation combinatorics. As we show in Appendix~\ref{section:cycle_count}, this identification reduces the problem to counting cycles of permutations of the form $\tau^{-1}c$, where $c$ is a uniformly random full cycle. This quantity can be computed exactly using known results, in particular those of \citet{dubach2024number}. While the argument in~\citet{dubach2024number} and ours is formulated in a different language, it is closely related to the literature on Hultman numbers and, in particular, to the corresponding closed formulas~\citep{stanley2011two,alexeev2011hultman}.

Putting these ingredients together yields an exact closed-form expression for the trace moment, and therefore for $Q_{n,k}$ and $S_{n,t}$. In this way, the proof reduces the finite-width signal propagation problem to a finite combinatorial counting problem that can be solved exactly.

\section{Depth--Width Scaling Regimes for $Q_{n,k}$ and $S_{n,t}$}\label{section:scaling_regimes}
The exact finite-width formulas in Theorem~\ref{thm:main} allow to analyze signal propagation under joint scaling of recurrent depth and width. In particular, they make it possible to characterize regimes such as $k\sim n^\alpha$ and $t\sim n^\alpha$, and thus identify the depth scales at which the infinite-width approximation remains accurate and those at which it fails.

For this purpose, it is convenient to rewrite $Q_{n,k}$ as follows:
\begin{equation}\label{eq:Q_product_form}
	Q_{n,k}
	=
	\dfrac{n}{(k+1)(k+2)}
	\Biggl[
	\prod_{j=0}^{k+1}\Bigl(1+\dfrac{j}{n}\Bigr)
	-
	\prod_{j=0}^{k+1}\Bigl(1-\dfrac{j}{n}\Bigr)
	\Biggr],
\end{equation}
since this form makes the relevant asymptotic regimes more transparent. The corresponding analysis for $Q_{n,k}$ is given in Appendix~\ref{section:Q_scaling}, and for $S_{n,t}$ in Appendix~\ref{section:S_scaling}. We now describe the asymptotic behavior of $Q_{n,k}$ and $S_{n,t}$ across the three regimes illustrated in Figure~\ref{fig:scaling_laws}: the \emph{subcritical}, \emph{critical}, and \emph{supercritical} regimes.

\subsection{Subcritical (Infinite-Width) Regime: $k=o(\sqrt{n})$ and $t=o(\sqrt{n})$}

We start with the regime in which the recurrent depth remains asymptotically below the critical $\sqrt n$ scale. In this case, signal propagation follows the same behavior as in the infinite-width limit:
\begin{equation}\label{eq:subcritical_regime_q_s}
	Q_{n,k} \to 1,
	\qquad
	S_{n,t} \sim t+1.
\end{equation}
Here, and throughout this section, the notation $a\sim b$ means that $a/b\to 1$ in the corresponding joint scaling limit. Precise statements and proofs for this regime are given in Theorems~\ref{thm:subcritical_regime} and~\ref{thm:subcritical_regime_s} in the Appendix. Below, we make several observations about the implications of these results.

\paragraph{Stability of Glorot initialization.}
The asymptotics in~\eqref{eq:subcritical_regime_q_s} coincide exactly with the classical infinite-width limit $n\to\infty$ at fixed $k$ and $t$. In particular, $Q_{n,k}$ remains asymptotically equal to $1$, uniformly over recurrent depths satisfying $k\ll \sqrt n$. Thus, under standard Glorot initialization~\eqref{eq:glorot}, signals remain stable throughout the entire subcritical regime. In this sense, our result extends the classical stability of Glorot initialization at infinite width with fixed recurrent depth~\citep{glorot2010understanding,chen2018dynamical} to the joint depth--width scaling regimes with $k=o(\sqrt n)$.

\paragraph{Infinite-width theory does not distinguish recurrent and feedforward networks.} Throughout the subcritical regime, recurrent and feedforward signal propagation remain asymptotically indistinguishable at the level of signal energies. More precisely, if one replaces the shared recurrent matrix $\W$ by independent matrices $\W^{(0)},\dots,\W^{(t)}$ across steps in the linear RNN~\eqref{eq:linear-rnn} or the LRU~\eqref{eq:lru}, then the asymptotic behavior of $Q_{n,k}$ and $S_{n,t}$ remains unchanged under the scaling $k,t=o(\sqrt n)$. This result also holds in existing infinite-width frameworks, such as Tensor Programs~\citep{yang2019scaling}, and this observation is made explicit for example in~\citet{alemohammad2021recurrent}.

While the second observation agrees with well-established literature, it is somewhat unsatisfactory from the perspective of recurrent modeling. In practice, recurrent networks are much more prone than feedforward ones to signal instability, and mathematically one expects repeated application of the same matrix to accumulate signal growth faster than independent layer-wise multiplication. The subcritical regime therefore identifies the range in which infinite-width theory remains faithful, but it also highlights its main limitation: below the critical scale, the theory cannot distinguish recurrence from feedforward composition. One of the main contributions of this paper is to show that this distinction becomes visible precisely at the critical scaling regime, which we analyze next.

\subsection{Critical Regime: $k \sim c\sqrt{n}$ and $t \sim \tilde c\sqrt{n}$}

The finite-width corrections become non-negligible when recurrent depth reaches the critical $\sqrt n$ scale. In this regime, the signal energy $Q_{n,k}$ no longer follows the infinite-width prediction and instead converges to a nontrivial critical profile, illustrated in Figure~\ref{fig:scaling_laws} and given by:
\begin{equation}\label{eq:q_critical}
	Q_{n,k}\to \frac{2}{c^2}\sinh\Bigl(\frac{c^2}{2}\Bigr), \quad \text{where } k \sim c\sqrt{n}, \quad c>0.
\end{equation}
Thus, unlike in the subcritical regime, the signal energy now depends nontrivially on the ratio $k/\sqrt n$. In particular, $Q_{n,k}$ grows strictly above its infinite-width value when $c>0$, showing systematic finite-width signal amplification over depth. This result is proven in Theorem~\ref{thm:critical_regime} in the Appendix. 

For the cumulative LRU signal energy $S_{n,t}$, this yields the following asymptotic law:
\begin{equation}\label{eq:s_critical}
	S_{n,t}
	\sim
	\sqrt{n}\int_0^{\tilde c}
	\dfrac{2}{x^2}\sinh \Bigl(\dfrac{x^2}{2}\Bigr)\,dx, \quad \text{where }  t\sim \tilde c\sqrt{n}, \quad \tilde c>0.
\end{equation}
where the integrand is continuously extended to take the value $1$ at $x=0$. This result is proven in Theorem~\ref{thm:critical_regime_s} in the Appendix. For large $\tilde c=t/\sqrt{n}$, this integral is dominated by its upper endpoint and behaves as $\tilde c^{-3}\exp(\tilde c^2/2)$. Thus, in the critical regime, signal growth becomes exponential in the scaling parameter $t/\sqrt n$. Below, we discuss the implications of the emergence of critical regime.

\paragraph{Instability of Glorot initialization.} The transition to the critical regime shows that standard Glorot initialization~\eqref{eq:glorot} becomes unstable once recurrent depth reaches the $\sqrt n$ scale. By Remark~\ref{remark:real_case}, this conclusion applies to both complex and real models. This helps explain why Glorot initialization, although natural from the infinite-width perspective, typically becomes unstable in modern long-range recurrent models~\citep{orvieto2023resurrecting,gu2021efficiently}. We note that \citet{bar2025revisiting} recently showed that the scaling $t\sim \tilde c\sqrt{n}$ already produces exponential growth in $S_{n,t}$. However, they established only a crude lower bound, which in particular missed the $\sqrt n$ prefactor in~\eqref{eq:s_critical}. Our result sharpens this picture in two ways: it shows that the $\sqrt n$ scale is exactly the threshold at which the transition to critical behavior occurs, and identifies the full asymptotic profile in this regime. In this sense, we give a precise answer to when Glorot ceases to be stable in long-range linear recurrences.

\paragraph{Recurrent versus feedforward signal propagation.} Starting at the critical regime, signal propagation in recurrent networks begins to differ sharply from that in feedforward networks---a distinction that is invisible in the infinite-width theory. While finite-width corrections accumulate and become non-negligible once depth reaches the $\sqrt n$ scale in recurrences, the relevant depth scale for fully-connected networks is $L\sim cn$~\citep{hanin2018neural,hanin2020products,seleznova2022neural}. Moreover, even in this proportional regime, the signal energies remain stable in feedforward networks, while finite-width effects are primarily reflected in their fluctuations, which has important consequences for convergence to the kernel regime~\citep{hanin2019finite,seleznova2022neural}. Thus, the effects of finite-width corrections and the depth scales at which they first become visible are fundamentally different in recurrent and feedforward architectures. This observation helps reconcile theory with practice: recurrent models are known to be substantially more prone to signal instability, but this distinction is not captured by standard infinite-width analyses.

\subsection{Supercritical sublinear regime: $\sqrt{n} \ll k \ll n$ and $\sqrt{n} \ll t \ll n$}

Beyond the critical depth scale, finite-width effects dominate, and both signal energies diverge in the joint $(n,k)$ limit. In particular, when $\sqrt{n} \ll k \ll n$ and $\sqrt{n} \ll t \ll n$, we have
\begin{equation}
	Q_{n,k}
	\sim
	\dfrac{n}{(k+1)(k+2)}e^{n\Psi(k/n)},
	\qquad
	S_{n,t}
	\sim
	\frac{n^2}{t^3}\exp\Bigl(n\Psi(t/n)\Bigr),
\end{equation}
where $\Psi(x):=(1+x)\log(1+x)-x$.  Thus, once recurrent depth exceeds the critical $\sqrt n$ scale, signal growth becomes exponential in depth and no finite joint scaling limit exists. These results are proven in Theorems~\ref{thm:supercritical_regime} and~\ref{thm:supercritical_regime_s} in the Appendix.

A particularly tractable subregime occurs when $k,t=o(n^{2/3})$.
In that case, the exponent can be expanded to the leading order, yielding the asymptotic profiles illustrated in Figure~\ref{fig:scaling_laws}:
\begin{equation}
	Q_{n,k}
	\sim
	\dfrac{n}{k^2}\exp\Bigl(\dfrac{k^2}{2n}\Bigr),
	\qquad
	S_{n,t}
	\sim
	\frac{n^2}{t^3}\exp\Bigl(\frac{t^2}{2n}\Bigr).
\end{equation}

In principle, the analysis can be extended further, beyond the proportional depth--width scaling. However, all such regimes exhibit rapid signal explosion, so we do not pursue them here.

\section{Numerical Experiments}\label{section:numerics}

We investigate two questions numerically: whether the derived theory accurately describes empirical signal propagation behavior, and if the identified phenomena persist beyond the linear setting.\footnote{Code and saved results available at: \url{https://github.com/mselezniova/long-range-rnn}.}

\paragraph{Linear recurrences.}
Figures~\ref{fig:numerics_C}--\ref{fig:numerics_R} compare Monte Carlo
estimates of $Q_{n,k}$ and $S_{n,t}$ against the critical scaling variable
$c=t/\sqrt n$. For complex Gaussian weights (Figure~\ref{fig:numerics_C}),
the curves at different widths collapse onto the critical profiles predicted
by~\eqref{eq:q_critical}--\eqref{eq:s_critical}. For real Gaussian weights
(Figure~\ref{fig:numerics_R}), the signal energies are larger but exhibit the
same qualitative amplification, consistent with
Remark~\ref{remark:real_case}.
\begin{figure}[H]
    \centering
    \includegraphics[width=\linewidth]{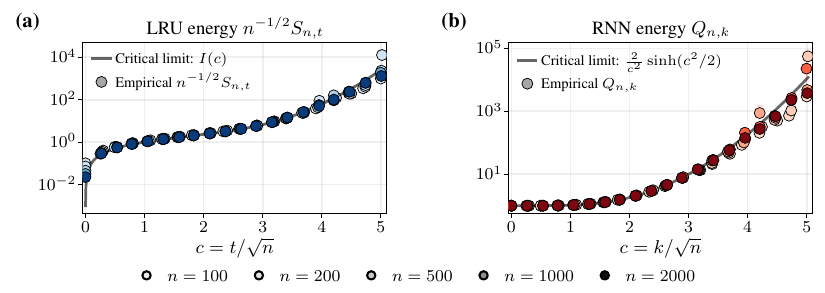}
   \caption{\textbf{Numerical validation of the critical scaling theory.} Monte Carlo estimates of $S_{n,t}$ (a) and $Q_{n,k}$ (b), based on $10^5$ samples and plotted against the critical scaling variable $c$ for several widths $n$. Inputs $\x$ are sampled uniformly on the sphere of radius $\sqrt n$, so that $\mathbb E[\x\x^\top]=\mathbb I_n$. In both panels, the numerical estimates are compared with the corresponding theoretical critical profiles, given by~\eqref{eq:s_critical} for $S_{n,t}$ and~\eqref{eq:q_critical} for $Q_{n,k}$. The close agreement illustrates convergence of the signal energies to their predicted critical limits. Error bars denote Monte Carlo standard errors of the mean.}
    \label{fig:numerics_C}
\end{figure}
\begin{figure}[H]
    \centering
    \includegraphics[width=\linewidth]{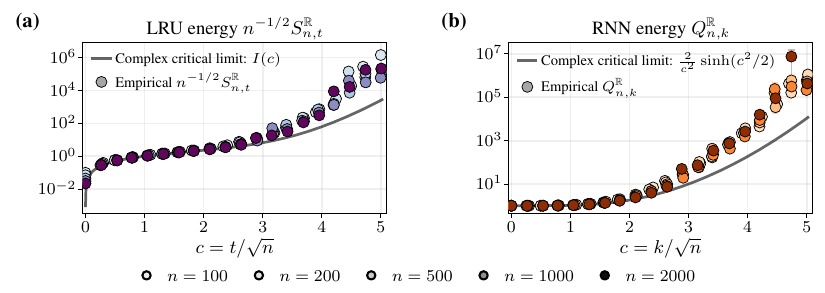}
    \caption{\textbf{Real Gaussian model.} Analogous to Figure~\ref{fig:numerics_C}, but for real Gaussian weights. In both panels, the solid curve is the complex critical profile from~\eqref{eq:s_critical} and~\eqref{eq:q_critical}. The figure illustrates that the real model signal energies are larger than their complex counterparts, in agreement with Remark~\ref{remark:real_case}.}
    \label{fig:numerics_R}
\end{figure}
\begin{figure}[H]
    \centering
    \includegraphics[width=\linewidth]{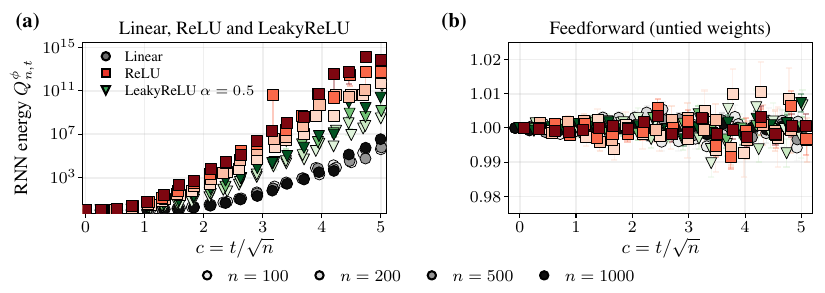}
    \caption{\textbf{Signal propagation in nonlinear RNNs.} Panel~(a) compares signal energies $Q^\phi_{n,t}$ for linear, ReLU, and LeakyReLU (slope $\alpha=0.5$) recurrences. Panel~(b) shows the corresponding feedforward models with independent weights $\W^{(t)}$ at each step $t$. Signal energies grow rapidly in the recurrent models but remain close to one in the feedforward models.}
    \label{fig:nonlinear}
\end{figure}

\paragraph{Nonlinear recurrences.}\label{section:nonlinear_numerics}
We next test whether finite-width signal amplification persists beyond the linear
setting. Namely, we compare the following recurrent models:
\begin{equation}\label{eq:nonlinear_rnn}
    \h_\phi^{(t+1)}=\phi(\W\h_\phi^{(t)}),
    \qquad
    \W_{ij}\sim\mathcal N(0,\sigma^2/n),
\end{equation}
where $\phi$ is linear, ReLU or LeakyReLU activation, and $\sigma^2$ is chosen according to the corresponding infinite-width variance-preserving initialization. Specifically, we use $\sigma^2=1$ for linear, $\sigma^2=2$ for ReLU, and $\sigma^2=2/(1+\alpha^2)$ for LeakyReLU with negative slope $\alpha$. Figure~\ref{fig:nonlinear}(a) compares the behavior of the signal energies $Q^\phi_{n,t}:= n^{-1}\mathbb E[\|\h_\phi^{(t)}\|^2_2]$ for the three activations. ReLU and LeakyReLU exhibit rapid signal growth, mirroring the effect observed for linear networks. Figure~\ref{fig:nonlinear}(b) plots signal energies of the corresponding feedforward models, where the weight matrices are drawn independently at each step $t$. In this setting, the signal energies remain close to one, as predicted by infinite-width theory. Together, these experiments illustrate how recurrent weight reuse drives the observed signal amplification.

\section{Discussion}\label{section:discussion}

Our results give a sharp characterization of when infinite-width signal propagation theory remains predictive in long-range recurrent networks. For recurrent depth $t=o(\sqrt n)$, signal energies follow their infinite-width predictions, while at the critical scale $t\sim\sqrt n$ a nontrivial joint scaling limit emerges. Thus, the $\sqrt n$ scale marks the transition at which finite-width effects become essential for describing signal propagation.

Importantly, the identified $\sqrt n$ scale is fundamentally a recurrent effect. In feedforward networks, where weights are independent across
layers, finite-width effects become non-negligible only at depths of order $n$. Thus, our results quantify how repeatedly reusing the same weights
across recurrent steps accelerates the accumulation of finite-width effects. Our experiments suggest that this mechanism persists beyond the linear
setting: ReLU and LeakyReLU recurrences exhibit pronounced signal growth, absent in the corresponding feedforward models.

More broadly, our results show that width alone does not determine when infinite-width theory is predictive: the relevant notion of ``large width''
depends on how depth grows with width and how parameters are shared across layers. Long-range recurrent models provide a particularly sharp
example of this principle, and offer a starting point for understanding finite-width effects in more
general models beyond fully-connected feedforward networks.

\subsection{Limitations and Future Work}\label{section:limitations}

Our analysis gives an exact finite-width characterization of linear recurrences with complex Gaussian weights and general input sequences,
providing a tractable starting point for several extensions.

A first natural direction is to broaden the class of weight distributions. Our Remark~\ref{remark:real_case} already shows that the
critical $\sqrt n$ depth scale leads to non-negligible deviations from infinite-width predictions for real
Gaussian weights. Obtaining full finite-width characterization or equally sharp asymptotics in the real case would further strengthen this. More generally, it would be interesting to study other initialization schemes, such as those used in modern linear recurrent architectures, including SSMs~\citep{gu2020hippo,gu2022parameterization}.

A second direction is to extend the theory beyond linear recurrences. Although many modern long-range sequence models, including LRUs and SSMs, are built
around linear recurrent blocks, classical recurrent architectures typically incorporate nonlinearities~\citep{gers2001lstm,cho2014learning}.
Our numerical experiments in Figure~\ref{fig:nonlinear} suggest that finite-width signal amplification persists under certain nonlinear activations. Characterizing signal propagation in nonlinear models under joint depth--width scaling theoretically is an natural direction for future work.

Finally, our results characterize signal propagation at random initialization. Extending the analysis to training could clarify how the finite-width effects identified here interact with the kernel (NTK)
regime known to describe training dynamics at infinite width~\citep{jacot2018neural,alemohammad2021recurrent}, and when increasing recurrent depth breaks convergence to this regime. This may also provide a route towards understanding feature learning under joint recurrent-depth--width scaling.

\section*{Acknowledgements}
This work was conducted as part of the Emmy Noether project “RMT4DL: Random Matrix Theory for Deep Learning.” Funded by the Deutsche Forschungsgemeinschaft (DFG, German Research Foundation) --- Project number 572957656. Numerical experiments for this research were conducted on the Lichtenberg high-performance computing cluster of TU Darmstadt.

\bibliographystyle{plainnat}
\bibliography{references}


\newpage
\appendix

\section{Related Work}\label{section:related_works}
We review related work on signal propagation theory, both at infinite width and under finite-width corrections and joint depth–width limits, and the implications for stability in recurrent models.

\subsection{Signal Propagation at Infinite Width}

Existing signal propagation theory primarily concerns the infinite-width limit $n\to\infty$ at fixed depth $L$. In fully-connected networks, the behavior of signal energies in this regime was derived by \citet{schoenholz2016deep} and \citet{poole2016exponential}. These analyses rely on the fact that preactivations in infinitely-wide networks become Gaussian, as shown rigorously e.g. in \citet{matthews2018gaussian}. The same perspective was subsequently extended to other architectures, including convolutional networks~\citep{xiao2018dynamical}, residual networks~\citep{tarnowski2019dynamical}, and recurrent networks~\citep{gilboa2019dynamical}. More broadly, infinite-width signal propagation theory has been formulated in a unified way for general deep learning architectures within the Tensor Programs framework~\citep{yang2019scaling,yang2020ntk,yang2020matrix,yang2021tensor}.

\paragraph{Recurrent models.} Several works explicitly analyzed signal propagation in vanilla RNNs and gated recurrent units in the infinite-width, fixed-depth setting~\citep{chen2018dynamical,gilboa2019dynamical,alemohammad2021recurrent}. Recurrent models also fall within the Tensor Programs framework. As mentioned above, a common conclusion of these works is that, at infinite width, signal propagation in recurrent models behaves as if the weights were independent across recurrent time steps. That is, at the level of infinite-width signal propagation, recurrence is indistinguishable from feedforward composition. This observation is made explicit, for example, in Theorem~3 of \citet{alemohammad2021recurrent} and also in \citet{segadlo2022unified}.

\paragraph{Stable initialization schemes.}
One of the main practical consequences of signal propagation theory is the principled design of stable initialization schemes, which prevent signals from exploding or vanishing with depth. This leads, in particular, to the familiar variance-preserving Gaussian initializations such as He initialization~\citep{he2015delving} and Glorot initialization~\citep{glorot2010understanding}.

\subsection{Finite-Width Corrections and Joint Depth--Width Limits}
A growing line of work studies how finite-width networks deviate from their infinite-width limits. A central question in this literature is how such deviations accumulate with depth, which naturally leads to the study of joint depth--width scaling limits in deep learning models.

\paragraph{Feedforward networks.} The most extensively studied setting here is fully-connected feedforward networks. In this case, finite-width corrections become non-negligible when depth scales linearly with width, i.e., in the regime $L\sim cn$. First results in this direction were obtained by \citet{hanin2018neural}, who analyzed how fluctuations of signal energies accumulate with depth. Follow-up works extended this picture to broader settings and studied implications for signal stability and convergence to the kernel regime~\citep{hanin2020products,hanin2019finite,hanin2024random,seleznova2022neural,yaida2020non}. A more general perspective on finite-width corrections was later developed in \citet{roberts2022principles}. Overall, in fully-connected feedforward networks, finite-width corrections can be analyzed layer by layer and are therefore comparatively tractable.

\paragraph{Recurrent models.} By contrast, the repeated reuse of the same matrix across time steps in recurrent models leads to a substantially more intricate structure of finite-width corrections and of their accumulation with depth. To the best of our knowledge, the only prior work to explicitly identify the $\sqrt n$ scale in linear recurrences is \citet{bar2025revisiting}, where the authors used finite-width corrections to the spectral radius of Gaussian matrices to show that deviations from infinite-width behavior accumulate on the $t\sim c\sqrt n$ scale. However, that argument yields only a crude lower bound on the signal energy and does not recover the sharp finite-width behavior derived in the present work. More broadly, several earlier works investigated fluctuations around the infinite-width limit in recurrent networks and observed that they are larger than in feedforward architectures~\citep{segadlo2022unified,grosvenor2022edge}. However, these works do not identify an explicit depth scale at which infinite-width predictions break down, nor do they provide a sharp characterization of the size of finite-width corrections.

\paragraph{Stability of Glorot initialization in recurrent sequence models.}
Because standard infinite-width signal propagation theory does not distinguish feedforward and recurrent models, it predicts that Glorot initialization~\citep{glorot2010understanding} should remain stable in long-range linear recurrences. However, in practice, recent works on recurrent sequence models have found that Glorot is unstable and have therefore introduced additional rescaling. For example, \citet{orvieto2023resurrecting} used truncated Glorot, which effectively contracts the spectrum by roughly~$0.8$, while \citet{gu2021efficiently} explicitly report iteratively reducing the Gaussian variance until the model no longer explodes. More recently, \citet{bar2025revisiting} proposed a theoretically motivated approach to such rescaling. At the same time, state-of-the-art linear recurrent sequence models typically rely on structured non-Gaussian initializations, while rescaled variants of Glorot are used mainly as baselines~\citep{gu2020hippo,fuhungry,gu2023mamba,gu2021efficiently}.

\section{Proof of Theorem~\ref{thm:main}}\label{section:proof_main}
The central part of the proof is the exact evaluation of the trace moment $\mathbb E \bigl [\tr\bigl((\W^k)^*\W^k\bigr) \bigr]$ for finite $n$ and $k$, given in the following Subsections~\ref{section:trace_moment} and \ref{section:cycle_count}. The remaining part of the proof, which reduces the signal energies $Q_{n,k}$ and $S_{n,t}$ to the trace moments is given in Subsection~\ref{section:reduction_to_trace}.

\subsection{Trace moment via Wick expansion and free-index counting}\label{section:trace_moment}

In this subsection, we carry out the standard reduction of the trace moment to a combinatorial counting problem. Namely, we expand the trace into monomials in the Gaussian entries of $\W$, apply Wick's theorem, and then count how many indices remain free after imposing the Kronecker constraints associated with each Wick pairing.

We begin by expanding the trace as follows:
\begin{align}
	\tr\bigl((\W^k)^* \W^k\bigr)
	&= \sum_{a_0,a_k=1}^n (\W^k)^*_{a_k a_0}(\W^k)_{a_0 a_k} \\
	&= \sum_{a_0,a_k=1}^n \overline{(\W^k)_{a_0 a_k}}\, (\W^k)_{a_0 a_k} \\
	&= \sum_{\substack{a_0,\dots,a_k=1\\ b_1,\dots,b_{k-1}=1}}^n
	\overline{\W_{a_0 a_1}\W_{a_1 a_2}\cdots \W_{a_{k-1} a_k}}\,
	\W_{a_0 b_1}\W_{b_1 b_2}\cdots \W_{b_{k-1} a_k} \\
	&= \sum_{\substack{a_0,\dots,a_k=1\\ b_1,\dots,b_{k-1}=1}}^n
	\overline{\W}_{a_0 a_1}\,\overline{\W}_{a_1 a_2}\cdots \overline{\W}_{a_{k-1} a_k}\,
	\W_{a_0 b_1}\W_{b_1 b_2}\cdots \W_{b_{k-1} a_k}.
\end{align}
Thus the trace moment reduces to expectations of monomials in the entries of $\W$, which are centered proper complex Gaussians with variance $1/n$ under the complex Glorot initialization~\eqref{eq:glorot}.

For such monomials, Wick's theorem takes a particularly simple form, specified in the next lemma.

\begin{lemma}[Wick theorem for proper complex Gaussians]\label{lemma:complex_wick}
	Let $(Z_1,\dots,Z_m)$ be a centered proper complex Gaussian vector, meaning that
	$\mathbb E[Z_i Z_j]=0$ for all $i,j\in\{1,\dots,m\}$.
	Let $I=(I_1,\dots,I_k)$ and $J=(J_1,\dots,J_r)$ be arbitrary index tuples with entries in $\{1,\dots,m\}$.
	Then
	\begin{equation}
		\mathbb E\!\left[\prod_{a=1}^k \overline{Z_{I_a}}\prod_{b=1}^r Z_{J_b}\right]
		=
		\delta_{k,r}\sum_{\sigma\in S_k}\prod_{a=1}^k
		\mathbb E\!\left[\overline{Z_{I_a}}\, Z_{J_{\sigma(a)}}\right],
	\end{equation}
	where $S_k$ denotes the set of permutations of $\{1,\dots,k\}$.
\end{lemma}

\begin{proof}
	We apply the usual Wick theorem to the family obtained by listing together the conjugated and non-conjugated Gaussian variables. Define the \emph{doubled family}
	\begin{equation}
		X_1,\dots,X_{k+r}
		:=
		(\overline{Z_{I_1}},\dots,\overline{Z_{I_k}},Z_{J_1},\dots,Z_{J_r}).
	\end{equation}
	Applying Wick's theorem to this family yields
	\begin{equation}\label{eq:wick_full}
		\mathbb E\!\left[\prod_{a=1}^k \overline{Z_{I_a}}\prod_{b=1}^r Z_{J_b}\right]
		=
		\sum_{\pi\in\mathcal P_2(k+r)}
		\prod_{\{p,q\}\in\pi}\mathbb E[X_pX_q],
	\end{equation}
	where $\mathcal P_2(k+r)$ denotes the set of all pairings of $\{1,\dots,k+r\}$.
	
	Since $(Z_1,\dots,Z_m)$ is proper complex Gaussian, we have
	\begin{equation}
		\mathbb E[Z_p Z_q]=0,
		\qquad
		\mathbb E[\overline{Z_p}\,\overline{Z_q}]
		=
		\overline{\mathbb E[Z_p Z_q]}
		=0
	\end{equation}
	for all $p,q\in\{1,\dots,m\}$. Therefore every pairing containing a pair of the form
	$(\overline{Z_p},\overline{Z_q})$ or $(Z_p,Z_q)$ contributes zero to the right-hand side of~\eqref{eq:wick_full}. Hence only pairings consisting entirely of mixed pairs
	$(\overline{Z_{I_a}}, Z_{J_b})$ can survive.
	
	Such pairings exist only if $k=r$. When $k=r$, every surviving pairing is uniquely determined by a permutation $\sigma\in S_k$, namely by matching
	\begin{equation}
		\overline{Z_{I_a}}
		\quad\text{with}\quad
		Z_{J_{\sigma(a)}},
		\qquad a=1,\dots,k.
	\end{equation}
	Substituting these surviving pairings into~\eqref{eq:wick_full} gives
	\begin{equation}
		\mathbb E\!\left[\prod_{a=1}^k \overline{Z_{I_a}}\prod_{b=1}^r Z_{J_b}\right]
		=
		\delta_{k,r}\sum_{\sigma\in S_k}\prod_{a=1}^k
		\mathbb E\!\left[\overline{Z_{I_a}}\, Z_{J_{\sigma(a)}}\right],
	\end{equation}
	as claimed.
\end{proof}

Applying Lemma~\ref{lemma:complex_wick} to the monomials appearing in the trace expansion, we see that the expectation is indexed by permutations $\sigma\in S_k$, each of which prescribes one possible pairing between the $k$ conjugated factors and the $k$ non-conjugated factors. The contribution of a fixed permutation $\sigma$ is therefore
\begin{equation}
	\sum_{\substack{a_0,\dots,a_k=1\\ b_1,\dots,b_{k-1}=1}}^n
	\prod_{r=1}^k \mathbb{E}\!\left[\overline \W_{a_{r-1}a_r} \, \W_{b_{\sigma(r)-1}b_{\sigma(r)}}\right].
\end{equation}
Since the entries of $\W$ are i.i.d.\ proper complex Gaussians with variance $1/n$, each covariance equals
\[
\mathbb{E}\!\left[\overline \W_{ij} \W_{pq}\right]
=
\frac{1}{n}\,\delta_{i,p}\delta_{j,q}.
\]
Hence the contribution of $\sigma$ becomes
\begin{equation}
	\sum_{\substack{a_0,\dots,a_k=1\\ b_1,\dots,b_{k-1}=1}}^n
	\prod_{r=1}^k \frac{1}{n}\delta_{a_{r-1},\,b_{\sigma(r)-1}}\delta_{a_r,\,b_{\sigma(r)}}.
\end{equation}
Here we use the convention
\[
b_0:=a_0,
\qquad
b_k:=a_k,
\]
so that all indices appearing in the contractions are among
\[
a_0,\dots,a_k,b_1,\dots,b_{k-1}.
\]

The key combinatorial quantity is the number of indices that remain unconstrained after all the Kronecker contractions have been imposed, as defined below.

\begin{definition}[Free-index count and defect of a permutation]
	For $\sigma\in S_k$, we define $F(\sigma)$ to be the number of \textbf{free indices}, namely the number of equivalence classes of indices that remain unconstrained after all $\delta$-contractions associated with $\sigma$ have been imposed. The corresponding \textbf{defect} is defined by
	\begin{equation}
		r(\sigma):=k+1-F(\sigma).
	\end{equation}
\end{definition}

Since each free index contributes a factor of $n$, whereas each covariance contributes a factor of $n^{-1}$, the total contribution of $\sigma$ is exactly
\begin{equation}
	n^{-k+F(\sigma)} = n^{1-r(\sigma)}.
\end{equation}
After dividing by the outer factor $1/n$, we obtain the following representation of the normalized trace moment.

\begin{lemma}[Trace moment as a free-index sum]\label{lemma:trace_permutation_sum}
	For every $k\geq 0$,
	\begin{equation}
		\mathbb{E}\!\left[\frac{1}{n}\tr\bigl((\W^k)^*\W^k\bigr)\right]
		=
		\sum_{\sigma\in S_k} n^{-r(\sigma)},
	\end{equation}
	where $r(\sigma)=k+1-F(\sigma)$ and $F(\sigma)$ is the number of free indices left after the Wick contractions corresponding to $\sigma$.
\end{lemma}

Lemma~\ref{lemma:trace_permutation_sum} reduces the problem to understanding the free-index count $F(\sigma)$, or equivalently the defect $r(\sigma)$, associated with each permutation $\sigma\in S_k$. In the next subsection, we show that $F(\sigma)$ can be identified with the number of cycles of a certain commutator of permutations, which allows the trace moment to be evaluated exactly.

\subsection{Exact evaluation of the trace moment via commutator cycle counts}\label{section:cycle_count}

By Lemma~\ref{lemma:trace_permutation_sum}, the trace moment is determined by the free-index count $F(\sigma)$ associated with each permutation $\sigma\in S_k$. We now show that this quantity admits a simple permutation-theoretic interpretation: it is exactly the number of cycles of a certain commutator. This reduces the trace computation to a cycle-counting problem, which can then be evaluated using known results on random permutations.

\begin{lemma}\label{lemma:free_indices_via_commutator}
	For a permutation $\sigma\in S_k$, the free-index count $F(\sigma)$ is exactly the number of cycles of
	\begin{equation}
		[\tau^{-1},\tilde\sigma^{-1}]
		:=
		\tau^{-1}\tilde\sigma^{-1}\tau\tilde\sigma,
	\end{equation}
	where $\tau=(0,1,\dots,k)\in S_{k+1}$ and $\tilde\sigma\in S_{k+1}$ is the extension of $\sigma$ fixing $0$, i.e.
	\begin{equation}
		\tilde\sigma(0)=0,
		\qquad
		\tilde\sigma(i)=\sigma(i),
		\quad i=1,\dots,k.
	\end{equation}
\end{lemma}

\begin{proof}
	Recall from the previous subsection that a permutation $\sigma\in S_k$ imposes the Kronecker constraints
	\begin{equation}
		a_{i-1}=b_{\sigma(i)-1},
		\qquad
		a_i=b_{\sigma(i)},
		\qquad i=1,\dots,k,
	\end{equation}
	together with the boundary conditions
	\[
	b_0=a_0,
	\qquad
	b_k=a_k.
	\]
	The free-index count $F(\sigma)$ is the number of equivalence classes generated by these identifications.
	
	To rewrite the constraints more symmetrically, set $j=\sigma(i)$. Then
	\begin{equation}
		a_{\sigma^{-1}(j)-1}=b_{j-1},
		\qquad
		a_{\sigma^{-1}(j)}=b_j,
		\qquad j=1,\dots,k.
	\end{equation}
	Eliminating the $b$-variables yields
	\begin{equation}
		a_{\sigma^{-1}(j-1)}=a_{\sigma^{-1}(j)-1},
		\qquad j=2,\dots,k,
	\end{equation}
	while the two boundary conditions become
	\[
	a_0=a_{\sigma^{-1}(1)-1},
	\qquad
	a_{\sigma^{-1}(k)}=a_k.
	\]
	
	These relations can be written uniformly by extending $\sigma$ to $\tilde\sigma\in S_{k+1}$ and introducing the full cycle $\tau=(0,1,\dots,k)$. Then for every $j\in\{0,\dots,k\}$,
	\begin{equation}
		a_{\tilde\sigma^{-1}(j)}
		=
		a_{\tau^{-1}\tilde\sigma^{-1}\tau(j)}.
	\end{equation}
	Indeed:
	\begin{itemize}
		\item for $j=0$, this reads $a_0=a_{\sigma^{-1}(1)-1}$;
		\item for $j=1,\dots,k-1$, this reads $a_{\sigma^{-1}(j)}=a_{\sigma^{-1}(j+1)-1}$;
		\item for $j=k$, this reads $a_{\sigma^{-1}(k)}=a_k$.
	\end{itemize}
	
	Now let
	\[
	m=\tilde\sigma^{-1}(j).
	\]
	Since $j=\tilde\sigma(m)$, the previous identity becomes
	\begin{equation}
		a_m
		=
		a_{\tau^{-1}\tilde\sigma^{-1}\tau\tilde\sigma(m)}
		=
		a_{[\tau^{-1},\tilde\sigma^{-1}](m)},
		\qquad m=0,\dots,k.
	\end{equation}
	Hence the equivalence classes among the indices $\{a_0,\dots,a_k\}$ are exactly the orbits of the permutation $[\tau^{-1},\tilde\sigma^{-1}]$. Therefore the number of free indices is precisely the number of cycles of this commutator.
\end{proof}

Lemma~\ref{lemma:free_indices_via_commutator} reduces the trace problem to counting cycles of the commutator $[\tau^{-1},\tilde\sigma^{-1}]$. The next step is to relate this commutator to a simpler random permutation model that has already been analyzed in the literature. Namely, we use the following known cycle-count formula, which is a reformulation of Theorem~2 in~\citet{dubach2024number}.

\begin{lemma}\label{lemma:cycles_count}
	Let $\tau=(0,\dots,M)\in S_{M+1}$, with $M\geq 0$, and let $c\in S_{M+1}$ be uniformly distributed among all $(M+1)$-cycles. Then the number of cycles of $\tau^{-1}c$, denoted $\mathcal{C}(\tau^{-1}c)$, satisfies
	\begin{equation}
		\mathbb{E}[t^{\mathcal{C}(\tau^{-1}c)}]
		=
		\dfrac{1}{(M+2)!}\bigl(F^+_{M+2}(t) - F^{-}_{M+2}(t)\bigr),
	\end{equation}
	where $F^{\pm}_{n}(t)$ are the rising and falling factorials, defined by
	\begin{equation}
		F^+_n(t) := t (t+1)\dots (t+n-1),
		\qquad
		F^-_n(t) := t (t-1)\dots (t-n+1).
	\end{equation}
\end{lemma}
We note that Theorem~2 in~\citet{dubach2024number} is stated in terms of the number of cycles of the commutator
\[
[\rho,\tau]:=\rho\tau\rho^{-1}\tau^{-1},
\]
where $\rho$ is uniformly distributed over $S_{M+1}$, rather than in terms of $\tau^{-1}c$ as in the formulation above. However, as also observed in~\citet{dubach2024number}, the two formulations are equivalent. Indeed, one can write
\begin{equation}
	[\rho,\tau] = \tau_1\tau_2,
	\qquad
	\tau_1 := \rho\tau\rho^{-1},
	\qquad
	\tau_2 := \tau^{-1}.
\end{equation}
Since $\tau_1$ is conjugate to $\tau$, it is uniformly distributed over the conjugacy class of $\tau$, that is, over all permutations with the same cycle structure as $\tau$. In our setting, $\tau$ is a single $(M+1)$-cycle, so $\tau_1$ is uniform among all $(M+1)$-cycles. Therefore the distribution of the number of cycles of $[\rho,\tau]$ is exactly the same as that of $\tau^{-1}c$, where $c$ is uniform among all $(M+1)$-cycles, which yields the stated reformulation.

It remains to show that the commutator in Lemma~\ref{lemma:free_indices_via_commutator} fits into this framework. The key observation is that conjugating the full cycle $\tau$ by a permutation fixing $0$ produces a uniformly random full cycle.

\begin{lemma}\label{lemma:c_uniform}
	Let $\tau=(0,\dots,k)$ and let $\tilde\sigma\in S_{k+1}$ be uniform on the set
	\begin{equation}\label{eq:set_H}
		H := \{\pi\in S_{k+1}: \pi(0)=0\}.
	\end{equation}
	Then $\tilde\sigma^{-1}\tau\tilde\sigma$ is uniform among all $(k+1)$-cycles in $S_{k+1}$.
\end{lemma}

\begin{proof}
	The set $H$ has cardinality $k!$, and there are also exactly $k!$ many $(k+1)$-cycles in $S_{k+1}$. It is therefore enough to show that the map
	\begin{equation}
		\phi_\tau: H \to \{(k+1)\text{-cycles in } S_{k+1}\},
		\qquad
		\phi_\tau(\pi)=\pi^{-1}\tau\pi,
	\end{equation}
	is injective.
	
	Suppose $\phi_\tau(\pi_1)=\phi_\tau(\pi_2)$ for some $\pi_1,\pi_2\in H$. Then
	\begin{equation}
		\pi_2\pi_1^{-1}\tau=\tau\pi_2\pi_1^{-1},
	\end{equation}
	so $\pi_2\pi_1^{-1}$ belongs to the centralizer of the full cycle $\tau$. But the centralizer of $\tau$ is exactly
	\[
	\{\tau^m: m=0,\dots,k\}.
	\]
	Among these elements, only the identity fixes $0$. Since both $\pi_1$ and $\pi_2$ fix $0$, so does $\pi_2\pi_1^{-1}$, and therefore $\pi_2\pi_1^{-1}=\id$. Hence $\pi_1=\pi_2$, proving injectivity. Thus $\phi_\tau$ is bijective, and the image of a uniform $\tilde\sigma\in H$ is uniform among all $(k+1)$-cycles.
\end{proof}

We can now combine the previous two lemmas to evaluate the cycle count appearing in Lemma~\ref{lemma:free_indices_via_commutator}.

\begin{lemma}\label{lemma:commutator_cycle_sum}
	Let $\tau=(0,\dots,k)$, let $\sigma$ range uniformly over $S_k$, and let $\tilde\sigma$ be its extension to $S_{k+1}$ fixing $0$. Then
	\begin{equation}
		\sum_{\sigma\in S_k} n^{\mathcal{C}([\tau^{-1}, \tilde\sigma^{-1}])}
		=
		\dfrac{k!}{(k+2)!}\bigl(F^+_{k+2}(n) - F^{-}_{k+2}(n)\bigr).
	\end{equation}
\end{lemma}

\begin{proof}
	Write
	\[
	c:=\tilde\sigma^{-1}\tau\tilde\sigma.
	\]
	Then
	\begin{equation}
		[\tau^{-1}, \tilde\sigma^{-1}]
		=
		\tau^{-1}\tilde\sigma^{-1}\tau\tilde\sigma
		=
		\tau^{-1}c.
	\end{equation}
	If $\sigma$ is uniform on $S_k$, then $\tilde\sigma$ is uniform on the set $H$ from~\eqref{eq:set_H}, and therefore $c$ is uniform among all $(k+1)$-cycles by Lemma~\ref{lemma:c_uniform}. It follows that
	\begin{equation}
		\frac{1}{|S_k|}\sum_{\sigma\in S_k} n^{\mathcal{C}([\tau^{-1}, \tilde\sigma^{-1}])}
		=
		\mathbb E\bigl[n^{\mathcal{C}(\tau^{-1}c)}\bigr].
	\end{equation}
	Now apply Lemma~\ref{lemma:cycles_count} with $M=k$ and $t=n$:
	\begin{equation}
		\mathbb E\bigl[n^{\mathcal{C}(\tau^{-1}c)}\bigr]
		=
		\dfrac{1}{(k+2)!}\bigl(F^+_{k+2}(n)-F^-_{k+2}(n)\bigr).
	\end{equation}
	Multiplying by $|S_k|=k!$ gives the claim.
\end{proof}

We are now ready to complete the trace computation.

\begin{proposition}\label{prop:trace_formula}
	For every $k\geq 0$,
	\begin{equation}
		\mathbb{E}\!\left[\frac{1}{n}\tr\bigl((\W^k)^*\W^k\bigr)\right]
		=
		\dfrac{F^+_{k+2}(n) - F^{-}_{k+2}(n)}{n^{k+1} (k+1) (k+2)}.
	\end{equation}
\end{proposition}

\begin{proof}
	By Lemma~\ref{lemma:trace_permutation_sum},
	\begin{equation}
		\mathbb{E}\!\left[\frac{1}{n}\tr\bigl((\W^k)^*\W^k\bigr)\right]
		=
		\sum_{\sigma\in S_k} n^{-r(\sigma)}.
	\end{equation}
	Using the definition of the defect and Lemma~\ref{lemma:free_indices_via_commutator}, this becomes
	\begin{equation}
		\mathbb{E}\!\left[\frac{1}{n}\tr\bigl((\W^k)^*\W^k\bigr)\right]
		=
		n^{-k-1}\sum_{\sigma\in S_k} n^{F(\sigma)}
		=
		n^{-k-1}\sum_{\sigma\in S_k} n^{\mathcal C([\tau^{-1},\tilde\sigma^{-1}])}.
	\end{equation}
	Applying Lemma~\ref{lemma:commutator_cycle_sum} yields
	\begin{equation}
		\mathbb{E}\!\left[\frac{1}{n}\tr\bigl((\W^k)^*\W^k\bigr)\right]
		=
		n^{-k-1}\cdot
		\dfrac{k!}{(k+2)!}\bigl(F^+_{k+2}(n) - F^{-}_{k+2}(n)\bigr),
	\end{equation}
	which simplifies to
	\begin{equation}
		\mathbb{E}\!\left[\frac{1}{n}\tr\bigl((\W^k)^*\W^k\bigr)\right]
		=
		\dfrac{F^+_{k+2}(n) - F^{-}_{k+2}(n)}{n^{k+1} (k+1) (k+2)}.
	\end{equation}
\end{proof}

In the next subsection, we show how Proposition~\ref{prop:trace_formula} yields the formulas for $Q_{n,k}$ and $S_{n,t}$ stated in Theorem~\ref{thm:main}.

\subsection{Reduction of $Q_{n,k}$ and $S_{n,t}$ to a trace moment}\label{section:reduction_to_trace}

We now return to the signal propagation quantities $Q_{n,k}$ and $S_{n,t}$ and express them in terms of the trace moments computed in the previous subsections, which was evaluated explicitly in Proposition~\ref{prop:trace_formula}.

\begin{lemma}\label{lemma:Q_trace}
	Consider a linear RNN (as defined in~\eqref{eq:linear-rnn}) with weight matrix $\W\in\mathbb{C}^{n\times n}$, initialized with i.i.d.\ proper complex Gaussian entries of variance $1/n$ (as in~\eqref{eq:glorot}). Then the RNN signal energy $Q_{n,k}$ (defined in~\eqref{eq:main_quantities}) satisfies
	\begin{equation}
		Q_{n,k}  =  \dfrac{1}{n}\tr(\Sigma^{(0)}) \cdot \frac{1}{n}\,\mathbb{E}\Bigl[ \tr\bigl( (\W^k)^*\W^k  \bigr)\Bigr],
	\end{equation}
	where $\Sigma^{(0)}:= \mathbb{E}\bigl[ \x^{(0)} (\x^{(0)})^\top \bigr]$ is the input covariance.
\end{lemma}

\begin{proof}
	We begin by rewriting the squared norm of the hidden state as a trace:
	\begin{equation}
		\|\h^{(k)}_{\textup{RNN}}\|_2^2
		=
		\| \W^k\x^{(0)}\|_2^2
		=
		\tr\Bigl(\x^{(0)} (\x^{(0)})^\top (\W^k)^*\W^k \Bigr).
	\end{equation}
	Taking expectations and using the independence between $\x^{(0)}$ and $\W$, we obtain
	\begin{equation}
		Q_{n,k}
		=
		\dfrac{1}{n}\mathbb{E}\bigl[ \|\h^{(k)}_{\textup{RNN}}\|_2^2 \bigr]
		=
		\dfrac{1}{n}\tr\Bigl(\mathbb{E}\bigl[\x^{(0)} (\x^{(0)})^\top\bigr] \, \mathbb{E}\bigl[(\W^k)^*\W^k \bigr]\Bigr).
	\end{equation}
	Thus it remains to understand the matrix
	\[
	\mathbb{E}\bigl[(\W^k)^*\W^k \bigr].
	\]
	
	We notice that this matrix is proportional to the identity, i.e.
	\begin{equation}
		\mathbb{E}\bigl[(\W^k)^*\W^k \bigr] = \alpha \mathbb{I}_n
	\end{equation}
	for some scalar $\alpha\in\mathbb{C}$. Indeed, by unitary invariance of the complex Gaussian law, for every fixed unitary matrix $U$ we have
	\begin{equation}
		U\W U^* \overset{d}{=} \W.
	\end{equation}
	Raising to the $k$-th power and taking adjoints gives
	\begin{equation}
		(\W^k)^*\W^k \overset{d}{=} U (\W^k)^*\W^k U^*.
	\end{equation}
	Taking expectations, we conclude that for every unitary $U$,
	\begin{equation}
		\mathbb{E}\bigl[(\W^k)^*\W^k\bigr]
		=
		U \mathbb{E}\bigl[(\W^k)^*\W^k \bigr] U^*.
	\end{equation}
	Equivalently,
	\begin{equation}
		\mathbb{E}\bigl[(\W^k)^*\W^k\bigr]\,U
		=
		U\,\mathbb{E}\bigl[(\W^k)^*\W^k \bigr],
	\end{equation}
	so $\mathbb{E}\bigl[(\W^k)^*\W^k\bigr]$ commutes with every unitary matrix. This is only possible if it is proportional to the identity.
	
	Substituting this back into the trace formula yields
	\begin{align}
		Q_{n,k}
		&=
		\dfrac{1}{n}\tr\bigl(\Sigma^{(0)}\,\alpha I_n\bigr)
		=
		\dfrac{1}{n}\tr(\Sigma^{(0)})\,\alpha.
	\end{align}
	On the other hand,
	\begin{equation}
		\tr\Bigl(\mathbb{E}\bigl[(\W^k)^*\W^k\bigr]\Bigr)=n\alpha,
	\end{equation}
	so
	\begin{equation}
		\alpha
		=
		\frac{1}{n}\tr\Bigl(\mathbb{E}\bigl[(\W^k)^*\W^k\bigr]\Bigr)
		=
		\frac{1}{n}\,\mathbb E\tr\bigl((\W^k)^*\W^k\bigr).
	\end{equation}
	Combining the two identities gives
	\begin{equation}
		Q_{n,k}
		=
		\dfrac{1}{n}\tr(\Sigma^{(0)}) \cdot \frac{1}{n}\,\mathbb{E}\Bigl[ \tr\bigl( (\W^k)^*\W^k  \bigr)\Bigr],
	\end{equation}
	as claimed.
\end{proof}

The next lemma gives the analogous reduction for the LRU signal energy. The only new point is that, after expanding the squared norm, one must show that the mixed terms with different powers of $\W$ vanish.

\begin{lemma}\label{lemma:S_trace}
	Consider an LRU (as defined in~\eqref{eq:lru}) with weight matrix $\W\in\mathbb{C}^{n\times n}$, initialized with i.i.d.\ proper complex Gaussian entries of variance $1/n$ (as in~\eqref{eq:glorot}). Then the LRU signal energy $S_{n,t}$ (defined in~\eqref{eq:main_quantities}) satisfies
	\begin{equation}
		S_{n,t} = \sum_{k=0}^t \dfrac{1}{n}\tr\!\bigl(\Sigma^{(t-k)}\bigr)\, \tilde Q_{n,k},
	\end{equation}
	where $\Sigma^{(k)} := \mathbb{E}\bigl[ \x^{(k)} (\x^{(k)})^\top \bigr]$ are the input covariances for all time steps $k=0,\dots,t$, and $\tilde Q_{n,k}$ is the RNN signal energy under whitened data, i.e.\ when $\Sigma^{(0)}=\mathbb{I}_n$.
\end{lemma}

\begin{proof}
	We start from the explicit representation of the LRU hidden state:
	\begin{equation}
		\h_{\textup{LRU}}^{(t)} = \sum_{k=0}^t \W^k \x^{(t-k)}.
	\end{equation}
	Expanding its squared norm gives
	\begin{align}
		\|\h_{\textup{LRU}}^{(t)}\|^2_2
		&=
		\sum_{k,r=0}^t  (\x^{(t-r)})^\top(\W^r)^*\W^k \x^{(t-k)} \\
		&=
		\sum_{k,r=0}^t \tr\Bigl(\x^{(t-k)} (\x^{(t-r)})^\top (\W^r)^*\W^k \Bigr).
	\end{align}
	Taking expectations and using the independence between the inputs and the weight matrix, we obtain
	\begin{equation}
		\mathbb E \bigl [\|\h_{\textup{LRU}}^{(t)}\|^2_2 \bigr]
		=
		\sum_{k,r=0}^t \tr\Bigl(\mathbb E\bigl[\x^{(t-k)} (\x^{(t-r)})^\top\bigr] \, \mathbb E\bigl[(\W^r)^*\W^k\bigr]\Bigr).
	\end{equation}
	
	We now show that the mixed terms with $k\neq r$ vanish. As in Lemma~\ref{lemma:Q_trace}, unitary invariance implies that $\mathbb E\bigl[(\W^r)^*\W^k\bigr]$ is proportional to identity. It therefore suffices to check that all diagonal entries are zero whenever $k\neq r$.
	
	Fix $i\in\{1,\dots,n\}$. Then
	\begin{align}
		\bigl[(\W^r)^*\W^k\bigr]_{ii}
		&=
		\sum_{a_0=1}^n \overline{(\W^r)_{a_0i}} \, (\W^k)_{a_0i} \\
		&=
		\sum_{\substack{a_0,\dots,a_{r-1}=1\\ b_1,\dots,b_{k-1}=1}}^n
		\overline{\W_{a_0a_1}} \overline{\W_{a_1 a_2}}\dots\overline{\W_{a_{r-1} i}}
		\W_{a_0b_1}\W_{b_1b_2}\dots\W_{b_{k-1}i}.
	\end{align}
	The family of Gaussian variables appearing here is proper complex Gaussian, so Lemma~\ref{lemma:complex_wick} applies. Since there are $r$ conjugated factors and $k$ non-conjugated factors, the expectation vanishes whenever $k\neq r$. Therefore,
	\begin{equation}
		\mathbb E\bigl[(\W^r)^*\W^k\bigr]=0
		\qquad\text{for all } k\neq r.
	\end{equation}
	
	Thus only the terms with $k=r$ remain, and the signal energy simplifies to
	\begin{equation}
		S_{n,t}
		=
		\dfrac{1}{n} \mathbb E \bigl [\|\h_{\textup{LRU}}^{(t)}\|^2_2 \bigr]
		=
		\dfrac{1}{n} \sum_{k=0}^t \tr\Bigl(\mathbb E\bigl[\x^{(t-k)} (\x^{(t-k)})^\top\bigr] \, \mathbb E\bigl[(\W^k)^*\W^k\bigr]\Bigr).
	\end{equation}
	At this point, we can apply exactly the same argument as in Lemma~\ref{lemma:Q_trace} to each term of the sum, obtaining
	\begin{equation}
		S_{n,t}
		=
		\sum_{k=0}^t  \frac{1}{n} \tr(\Sigma^{(t-k)})\cdot \frac{1}{n}\,\mathbb{E}\Bigl[ \tr\bigl( (\W^k)^*\W^k  \bigr) \Bigr].
	\end{equation}
	By definition, the quantity
	\begin{equation}
		\tilde Q_{n,k}
		=
		\frac{1}{n}\,\mathbb{E}\Bigl[ \tr\bigl( (\W^k)^*\W^k  \bigr) \Bigr]
	\end{equation}
	is precisely the RNN signal energy corresponding to whitened inputs. Substituting this into the previous expression gives
	\begin{equation}
		S_{n,t} = \sum_{k=0}^t  \frac{1}{n} \tr\!\bigl(\Sigma^{(t-k)}\bigr)\, \tilde Q_{n,k},
	\end{equation}
	as claimed.
\end{proof}

Finally, Proposition~\ref{prop:trace_formula} can now be substituted into Lemmas~\ref{lemma:Q_trace} and~\ref{lemma:S_trace}, which yields the explicit formulas for $Q_{n,k}$ and $S_{n,t}$ stated in Theorem~\ref{thm:main}.

\subsection{Real case: Proof of Remark~\ref{remark:real_case}}\label{section:real_case}
\begin{proof}
	We first treat the RNN quantity $Q_{n,k}$. Let $\W^{\mathbb R}$ denote a real Gaussian matrix with i.i.d.\ entries of variance $1/n$, and let
	\[
	Q_{n,k}^{\mathbb R}
	=
	\frac{1}{n}\,\mathbb E\bigl[\|\left(\W^{\mathbb R}\right)^k \x^{(0)}\|_2^2\bigr].
	\]
	As in Lemma~\ref{lemma:Q_trace}, orthogonal invariance implies that
	\[
	\mathbb E\!\left[\bigl((\W^{\mathbb R})^k\bigr)^\top (\W^{\mathbb R})^k\right]
	=
	\alpha_k^{\mathbb R} I_n
	\]
	for some scalar $\alpha_k^{\mathbb R}\ge 0$, and therefore
	\[
	Q_{n,k}^{\mathbb R}
	=
	\frac{1}{n}\tr(\Sigma^{(0)})\,
	\frac{1}{n}\mathbb E\!\left[\tr\bigl(((\W^{\mathbb R})^k)^\top (\W^{\mathbb R})^k\bigr)\right].
	\]
	So it suffices to compare the corresponding trace moments.

	Expanding the trace exactly as in Appendix~\ref{section:trace_moment}, we obtain a sum of monomials in the entries of $\W$. In the complex case, by Lemma~\ref{lemma:complex_wick}, only pairings between conjugated and non-conjugated factors contribute. In the real Gaussian case, Wick's theorem sums instead over \emph{all} pairings of the $2k$ factors. Since
	\[
	\mathbb E[\,\W^{\mathbb R}_{ij}\W^{\mathbb R}_{pq}\,]
	=
	\frac{1}{n}\delta_{i,p}\delta_{j,q}\ge 0,
	\]
	every contributing pairing in the real case gives a nonnegative contribution. Moreover, the pairings that survive in the complex case form a subset of the pairings appearing in the real case. Hence
	\[
	\mathbb E\!\left[\frac{1}{n}\tr\bigl(((\W^{\mathbb R})^k)^\top (\W^{\mathbb R})^k\bigr)\right]
	\ge
	\mathbb E\!\left[\frac{1}{n}\tr\bigl((\W^k)^*\W^k\bigr)\right],
	\]
	and therefore
	\[
	Q_{n,k}^{\mathbb R}\ge Q_{n,k}.
	\]

	We now turn to the LRU quantity $S_{n,t}^{\mathbb R}$. Expanding the squared norm as in Lemma~\ref{lemma:S_trace}, we obtain
	\[
	S_{n,t}^{\mathbb R}
	=
	\frac{1}{n}\sum_{k,r=0}^t
	\tr\!\Bigl(
	\mathbb E[\x^{(t-k)}(\x^{(t-r)})^\top]\,
	\mathbb E[((\W^{\mathbb R})^r)^\top (\W^{\mathbb R})^k]
	\Bigr).
	\]
	Again by orthogonal invariance,
	\[
	\mathbb E[((\W^{\mathbb R})^r)^\top (\W^{\mathbb R})^k]
	=
	\alpha_{k,r}^{\mathbb R} I_n
	\]
	for some scalar $\alpha_{k,r}^{\mathbb R}\ge 0$. Indeed,
	\[
	\alpha_{k,r}^{\mathbb R}
	=
	\frac{1}{n}\mathbb E\!\left[\tr\bigl(((\W^{\mathbb R})^r)^\top (\W^{\mathbb R})^k\bigr)\right],
	\]
	and this trace moment is again a sum of nonnegative Wick-pairing contributions. Therefore
	\[
	S_{n,t}^{\mathbb R}
	=
	\sum_{k,r=0}^t
	\frac{1}{n} \tr\!\bigl(\mathbb E[\x^{(t-k)}(\x^{(t-r)})^\top]\bigr)\,
	\alpha_{k,r}^{\mathbb R}.
	\]
	Under the assumption
	\[
	\tr\!\bigl(\mathbb E[\x^{(s)}(\x^{(r)})^\top]\bigr)\ge 0
	\qquad\text{for all } 0\le s,r\le t,
	\]
	each summand is nonnegative. In particular, restricting to the diagonal terms $k=r$ gives
	\[
	S_{n,t}^{\mathbb R}
	\ge
	\sum_{k=0}^t
	\frac{1}{n}\tr\!\bigl(\mathbb E[\x^{(t-k)}(\x^{(t-k)})^\top]\bigr)\,
	\alpha_{k,k}^{\mathbb R}.
	\]
	By the already established inequality for the diagonal trace moments,
	\[
	\alpha_{k,k}^{\mathbb R}\ge \tilde Q_{n,k},
	\]
	where $\tilde Q_{n,k}$ is the corresponding complex RNN signal energy under whitened inputs. Hence
	\[
	S_{n,t}^{\mathbb R}
	\ge
	\sum_{k=0}^t\frac{1}{n}
	\tr\!\bigl(\mathbb E[\x^{(t-k)}(\x^{(t-k)})^\top]\bigr)\,
	\tilde Q_{n,k}
	=
	S_{n,t}.
	\]

	This proves
	\[
	Q_{n,k}^{\mathbb R}\ge Q_{n,k},
	\qquad
	S_{n,t}^{\mathbb R}\ge S_{n,t}.
	\]
	The final claim follows immediately: if the complex model exhibits exploding signal energy in a given depth--width scaling regime, then so does the real model.
\end{proof}

\section{Proofs for Scaling Regimes of $Q_{n,k}$}\label{section:Q_scaling}

This section contains the proofs of the asymptotic statements for the signal energy $Q_{n,k}$ in the three joint $(k,n)$ scaling regimes identified in the main text: subcritical, critical, and supercritical. Each proof is obtained by a direct asymptotic analysis of the exact finite-width formula for $Q_{n,k}$ given in Theorem~\ref{thm:main}.

\begin{theorem}[Subcritical regime]\label{thm:subcritical_regime} If $k=o(\sqrt{n})$, we have
\begin{equation}
    Q_{n,k} \to 1.
\end{equation}
\end{theorem}
\begin{proof}
    Let us define $P_{\pm} = \prod_{j=0}^{k+1}(1\pm j/n)$, then we have:
    \begin{equation}
        Q_{n,k} = \dfrac{n}{(k+1)(k+2)}[P_+ - P_-].
    \end{equation}
    As $k\ll n$, we can get the following by Taylor expansion of the logarithm:
    \begin{equation}
        \log P_{\pm} = \sum_{j=0}^{k+1}\log\Bigl(1\pm\dfrac{j}{n}\Bigr) = \sum_{j=0}^{k+1}\Bigl( \pm\dfrac{j}{n}+ O\Bigl(\dfrac{j^2}{n^2}\Bigr)\Biggr).
    \end{equation}
    Under $k=o(\sqrt{n})$, we have the following for the terms above:
    \begin{equation}
        a_{n,k}:=\sum_{j=0}^{k+1}\dfrac{j}{n} = \dfrac{(k+1)(k+2)}{2n} = O\Bigl(\dfrac{k^2}{n}\Bigr), \quad  b_{n,k} := \sum_{j=0}^{k+1}O\Bigl(\dfrac{j^2}{n^2}\Bigr) = O\Bigl(\dfrac{k^3}{n^2}\Bigr) = o(a_{n,k}).
    \end{equation}
    From here, we can write $P_{\pm} = e^{\pm a_{n,k}}(1+ o(a_{n,k}))$. Further taking into account that $a_{n,k}=o(1)$ and thus $e^{\pm a_{n,k}} = 1+ o(1)$, we can write:
    \begin{equation}
        P_+ - P_- = (e^{a_{n,k}} - e^{-a_{n,k}}) + o(a_{n,k}) = 2a_{n,k} + o(a_{n,k}) = \dfrac{(k+1)(k+2)}{n}(1 + o(1)).
    \end{equation}
    Plugging this result into $Q_{n,k}$, we get:
    \begin{equation}
        Q_{n,k} = 1 + o(1) \to 1 
    \end{equation}
    for large $k,n$ whenever $k=o(\sqrt{n})$.
\end{proof}

\begin{theorem}[Critical regime]\label{thm:critical_regime} 
    If $k/\sqrt{n} \to c$ for some constant $c>0$, we have
\begin{equation}
    Q_{n,k} \to \dfrac{2}{c^2}\sinh\Bigl(\dfrac{c^2}{2}\Bigr).
\end{equation}
\end{theorem}
\begin{proof}
    Analogous to the proof of Theorem~\ref{thm:subcritical_regime}, we can write:
    \begin{equation}
        P_+ - P_- = (e^{a_{n,k}} - e^{-a_{n,k}}) + o(1).
    \end{equation}
    When $k/\sqrt{n} \to c$, we then have:
    \begin{equation}
        a_{n,k}\to \dfrac{c^2}{2}, \quad P_+ - P_- = e^{c^2/2} - e^{-c^2/2} = 2\sinh\Bigl(\dfrac{c^2}{2}\Bigr).
    \end{equation}
    Then for $Q_{n,k}$ we get the following:
    \begin{equation}
        Q_{n,k} = \dfrac{n}{(k+1)(k+2)}(P_+ -P_-)\to\dfrac{2}{c^2}\sinh\Bigl(\dfrac{c^2}{2}\Bigr).
    \end{equation}
\end{proof}

\begin{theorem}[Supercritical sublinear regime]\label{thm:supercritical_regime} If $\sqrt{n}\ll k \ll n$ or, in other words, $k/\sqrt{n}\to\infty$ and $k/n\to 0$, then 
\begin{equation}
    Q_{n,k} \sim \dfrac{n}{(k+1)(k+2)}e^{n\Psi(k/n)},
\end{equation}
where $\Psi(x):=(1+x)\log(1+x) - x$. In particular, this regime holds when $k \sim n^\alpha$ for $\alpha\in(1/2,1)$.
\end{theorem}
\begin{proof}
    First, we notice that in this regime $P_-$ is negligible in comparison with $P_+$. Indeed, when $k\ll n$ we can write:
    \begin{equation}
        \log\dfrac{P_-}{P_+} = \sum_{j=0}^{k+1}\log\dfrac{1-j/n}{1+j/n} = - \dfrac{2}{n} \sum_{j=0}^{k+1}j + O\Bigl(\dfrac{1}{n^3} \sum_{j=0}^{k+1}j^3\Bigr) = -\dfrac{(k+1)(k+2)}{n} + O\Bigl(\dfrac{k^4}{n^3}\Bigr)
    \end{equation}
    Further, we have $O(k^4/n^3) = o(k^2/n)$ when $k=o(n)$, and therefore under $k^2/n\to\infty$ we get:
    \begin{equation}
        \log\dfrac{P_-}{P_+} = - \dfrac{k^2}{n}(1+o(1)) \to -\infty, \quad P_+-P_- = P_+(1+o(1)).
    \end{equation}
It remains to estimate $P_+$, which we will do via approximating the sum with an integral:
\begin{equation}
    \log P_+ = \sum_{j=0}^{k+1}\log \Bigl(1+\dfrac{j}{n}\Bigr) = n \int_0^{(k+1)/n} \log(1+x)dx + O(\log(1+(k+1)/n)).
\end{equation}
Taking into account that $O(\log(1+(k+1)/n)) = O(k/n) = o(1)$ and computing the integral, we get:
\begin{equation}
    \log P_+ = n \Psi\Bigl(\dfrac{k+1}{n}\Bigr) + o(1) = n\Psi\Bigl(\dfrac{k}{n}\Bigr) + o(1).
\end{equation}
Combining everything, we get:
\begin{equation}
    P_+ - P_- = \exp(n\Psi(k/n) + o(1)),
\end{equation}
which results in the final expression for $Q_{n,k}$:
\begin{equation}
    Q_{n,k} = \dfrac{n}{(k+1)(k+2)}\exp(n\Psi(k/n) + o(1)).
\end{equation}
\end{proof}

\section{Proofs for Scaling Regimes of $S_{n,t}$}\label{section:S_scaling}

This section proves the asymptotic behavior of the signal energy $S_{n,t}$ in the three joint $(t,n)$ scaling regimes identified in the main text: subcritical, critical, and supercritical. The arguments combine the exact finite-width representation of $S_{n,t}$ from Theorem~\ref{thm:main} with the asymptotic behavior of the corresponding RNN signal energies $Q_{n,k}$ established in Appendix~\ref{section:Q_scaling}.

\begin{theorem}[Subcritical regime]\label{thm:subcritical_regime_s}
    If $t=o(\sqrt{n})$, we have:
    \begin{equation}
        S_{n,t} \sim t+1.
    \end{equation}
\end{theorem}
\begin{proof}
    Under $t=o(\sqrt{n})$, we have 
    \begin{equation}
        Q_{n,k} = 1 + o(1)
    \end{equation}
    uniformly over $k\leq t$ by Theorem~\ref{thm:subcritical_regime}. This directly implies the claim:
    \begin{equation}
        S_{n,t} = \sum_{k=0}^t Q_{n,k} = (t+1) ( 1 + o(1)).
    \end{equation}
\end{proof}

\begin{theorem}[Critical regime]\label{thm:critical_regime_s}
    If $t/\sqrt{n}\to c$ for some constant $c>0$, we have
    \begin{equation}
        S_{n,t}\sim  \sqrt{n}\int_0^c q(x)\,dx,
    \end{equation}
    where
    \begin{equation}
        q(x):=
        \begin{cases}
            1, & x=0,\\[1ex]
            \dfrac{2}{x^2}\sinh\Bigl(\dfrac{x^2}{2}\Bigr), & x>0.
        \end{cases}
    \end{equation}
\end{theorem}

\begin{proof}
    We first show that, in the critical window, $Q_{n,k}$ admits a uniform approximation by the profile $q(k/\sqrt{n})$ for all $k\leq t$. As before, we rewrite $Q_{n,k}$ via $P_{\pm}$ as follows:
    \begin{equation}
        P_{\pm}(n,k):=\prod_{j=0}^{k+1}\Bigl(1\pm \frac{j}{n}\Bigr),
        \qquad
        Q_{n,k}=\frac{n}{(k+1)(k+2)}\bigl(P_+(n,k)-P_-(n,k)\bigr).
    \end{equation}
    Similarly to the proof of Theorem~\ref{thm:critical_regime}, we define
    \begin{equation}
        a_{n,k}:=\frac{1}{n}\sum_{j=0}^{k+1}j=\frac{(k+1)(k+2)}{2n},
        \qquad
        b_{n,k}:=\frac{1}{n^2}\sum_{j=0}^{k+1}j^2
        =\frac{(k+1)(k+2)(2k+3)}{6n^2}.
    \end{equation}
    Under $k=O(\sqrt{n})$, we have uniformly in $k\leq t$:
    \begin{equation}
        a_{n,k}=O(1), \qquad b_{n,k}=O(n^{-1/2}).
    \end{equation}
    Therefore, by the same logarithmic expansion as before,
    \begin{equation}
        \log P_{\pm}(n,k)=\pm a_{n,k}+O(b_{n,k})
    \end{equation}
    uniformly over $k\leq t$, and exponentiating yields
    \begin{equation}
        P_+(n,k)-P_-(n,k)=2\sinh(a_{n,k})+O(b_{n,k}).
    \end{equation}
    Plugging this into the formula for $Q_{n,k}$, we obtain
    \begin{equation}
        Q_{n,k}
        =
        \frac{2n}{(k+1)(k+2)}\sinh(a_{n,k})
        +O\Bigl(\frac{n\,b_{n,k}}{(k+1)(k+2)}\Bigr).
    \end{equation}
    Since
    \begin{equation}
        \frac{n\,b_{n,k}}{(k+1)(k+2)}
        =
        \frac{2k+3}{6n}
        =
        O(n^{-1/2})
    \end{equation}
    uniformly for $k\leq t$, it follows that
    \begin{equation}
        Q_{n,k}
        =
        \frac{2n}{(k+1)(k+2)}\sinh(a_{n,k})+O(n^{-1/2}).
    \end{equation}
    Next, define
    \begin{equation}
        x_{n,k}:=\frac{k}{\sqrt{n}}.
    \end{equation}
    Then
    \begin{equation}
        a_{n,k}
        =
        \frac{(k+1)(k+2)}{2n}
        =
        \frac{k^2}{2n}+\frac{3k+2}{2n}
        =
        \frac{x_{n,k}^2}{2}+O(n^{-1/2})
    \end{equation}
    uniformly for $k\leq t$. Since the function
    \begin{equation}
        g(u):=
        \begin{cases}
            1, & u=0,\\[1ex]
            \dfrac{\sinh(u)}{u}, & u>0
        \end{cases}
    \end{equation}
    is continuous on bounded intervals, we obtain
    \begin{equation}
        \frac{2n}{(k+1)(k+2)}\sinh(a_{n,k})
        =
        g(a_{n,k})
        =
        g\Bigl(\frac{x_{n,k}^2}{2}\Bigr)+o(1)
    \end{equation}
    uniformly over $k\leq t$. In other words,
    \begin{equation}
        \sup_{0\leq k\leq t}\left|Q_{n,k}-q\Bigl(\frac{k}{\sqrt{n}}\Bigr)\right|\to 0,
    \end{equation}
    where
    \begin{equation}
        q(x)=g\Bigl(\frac{x^2}{2}\Bigr)=
        \begin{cases}
            1, & x=0,\\[1ex]
            \dfrac{2}{x^2}\sinh\Bigl(\dfrac{x^2}{2}\Bigr), & x>0.
        \end{cases}
    \end{equation}

    We now sum over $k\leq t$. Since $t/\sqrt{n}\to c$, we have $t=O(\sqrt{n})$, hence
    \begin{equation}
        S_{n,t}
        =
        \sum_{k=0}^t q\Bigl(\frac{k}{\sqrt{n}}\Bigr)+o(\sqrt{n}).
    \end{equation}
    Dividing by $\sqrt{n}$ gives
    \begin{equation}
        \frac{S_{n,t}}{\sqrt{n}}
        =
        \frac{1}{\sqrt{n}}\sum_{k=0}^t q\Bigl(\frac{k}{\sqrt{n}}\Bigr)+o(1).
    \end{equation}
    Since $q$ is continuous on $[0,c+\varepsilon]$ for any fixed $\varepsilon>0$, the sum is a Riemann sum, and therefore
    \begin{equation}
        \frac{1}{\sqrt{n}}\sum_{k=0}^t q\Bigl(\frac{k}{\sqrt{n}}\Bigr)\to \int_0^c q(x)\,dx.
    \end{equation}
    This proves that
    \begin{equation}
        \frac{S_{n,t}}{\sqrt{n}} \to \int_0^c q(x)\,dx.
    \end{equation}
\end{proof}

\begin{theorem}[Supercritical sublinear regime]\label{thm:supercritical_regime_s}
    If $t/\sqrt{n}\to\infty$ and $t/n\to 0$, then
    \begin{equation}
        S_{n,t}\sim \frac{Q_{n,t}}{\log(1+t/n)}.
    \end{equation}
    Equivalently,
    \begin{equation}
        S_{n,t}\sim \frac{n^2}{t^3}\exp\Bigl(n\Psi(t/n)\Bigr),
    \end{equation}
    where $\Psi(x):=(1+x)\log(1+x)-x$.
\end{theorem}

\begin{proof}
    Recall from Theorem~\ref{thm:main} that
    \begin{equation}
        S_{n,t}=\sum_{k=0}^t Q_{n,k}.
    \end{equation}
    We will show that this sum is asymptotically dominated by an endpoint window near $k=t$. Specifically, we choose an integer sequence $L_n$ satisfying
\begin{equation}
    \frac{n}{t}\ll L_n\ll \sqrt{n}.
\end{equation}
Such a sequence exists because $t/\sqrt n\to\infty$. For $\delta_n:=\log(1+t/n)\sim t/n$, this choice implies
\begin{equation}
L_n\delta_n\to\infty,\qquad
L_n=o(t).
\end{equation}
We then split the sum into the endpoint window and its complement as follows:
\begin{equation}
    S_{n,t}
    =
    \sum_{k=0}^{t-L_n-1}Q_{n,k}
    +
    \sum_{k=t-L_n}^{t}Q_{n,k}.
\end{equation}

    We first analyze the endpoint window. For $k=t-j$ with $0\le j\le L_n$, we have $j=o(t)$, hence $k\sim t$. Therefore, uniformly over $0\le j\le L_n$,
    \begin{equation}
        \frac{k}{\sqrt n}\to\infty,
        \qquad
        \frac{k}{n}\to0.
    \end{equation}
    The proof of Theorem~\ref{thm:supercritical_regime}, applied uniformly on this window, gives
    \begin{equation}\label{eq:local_uniform_Q}
        Q_{n,t-j}
        =
        f_n(t-j)(1+o(1)),
        \qquad
        0\le j\le L_n,
    \end{equation}
    where we defined
    \begin{equation}
        f_n(k):=
        \frac{n}{(k+1)(k+2)}
        \exp\Bigl(n\Psi(k/n)\Bigr).
    \end{equation}
    We briefly justify this local uniformity. Writing
    \begin{equation}
        P_{\pm}(n,k):=
        \prod_{j=0}^{k+1}\left(1\pm\frac{j}{n}\right),
    \end{equation}
    the exact formula for $Q_{n,k}$ gives
    \begin{equation}
        Q_{n,k}
        =
        \frac{n}{(k+1)(k+2)}
        \bigl(P_+(n,k)-P_-(n,k)\bigr).
    \end{equation}
    Uniformly for $k=t-j$, $0\le j\le L_n$, we have
    \begin{equation}
        \log\frac{P_-(n,k)}{P_+(n,k)}
        =
        -\frac{k^2}{n}(1+o(1))
        \to -\infty,
    \end{equation}
    because $k\sim t$ and $t^2/n\to\infty$. Thus
    \begin{equation}
        P_-(n,k)=o(P_+(n,k))
    \end{equation}
    uniformly on the endpoint window. Moreover, by the same Riemann-sum estimate as in Theorem~\ref{thm:supercritical_regime},
    \begin{equation}
        \log P_+(n,k)
        =
        n\Psi(k/n)+o(1)
    \end{equation}
    uniformly for $k=t-j$, $0\le j\le L_n$, since all such $k$ satisfy $k\sim t$ and $t/n\to0$. This proves \eqref{eq:local_uniform_Q}.

    We now evaluate the endpoint contribution. For $0\le j\le L_n$,
    \begin{align}
        \frac{f_n(t-j)}{f_n(t)}
        &=
        \frac{(t+1)(t+2)}{(t-j+1)(t-j+2)}
        \exp\Bigl(
        -n\bigl[\Psi(t/n)-\Psi((t-j)/n)\bigr]
        \Bigr).
    \end{align}
    Since $j\le L_n=o(t)$, the prefactor satisfies
    \begin{equation}
        \frac{(t+1)(t+2)}{(t-j+1)(t-j+2)}
        =
        1+o(1)
    \end{equation}
    uniformly for $0\le j\le L_n$. Using the Taylor expansion of $\Psi((t-j)/n)$ around $t/n$, we get
\begin{equation}
    \Psi((t-j)/n)
    =
    \Psi(t/n)
    -
    \frac{j}{n}\Psi'(t/n)
    +
    O\left(\frac{j^2}{n^2}\right),
\end{equation}
where the error term is uniform over $0\le j\le L_n$, since
$\Psi''(x)=1/(1+x)$ is bounded on $[0,t/n]$ for large $n$. Therefore,
\begin{equation}
    n\bigl[\Psi(t/n)-\Psi((t-j)/n)\bigr]
    =
    j\Psi'(t/n)
    +
    O\left(\frac{j^2}{n}\right).
\end{equation}
Since $\Psi'(x)=\log(1+x)$ and $\delta_n:=\log(1+t/n)$, this gives
\begin{equation}
    n\bigl[\Psi(t/n)-\Psi((t-j)/n)\bigr]
    =
    j\delta_n
    +
    O\left(\frac{j^2}{n}\right).
\end{equation}
By the choice of the endpoint window, $L_n^2/n\to0$, we then have:
\begin{equation}
    n\bigl[\Psi(t/n)-\Psi((t-j)/n)\bigr]
    =
    j\delta_n+o(1)
\end{equation}
uniformly over $0\le j\le L_n$. Therefore,
\begin{equation}
    \frac{f_n(t-j)}{f_n(t)}
    =
    e^{-j\delta_n}(1+o(1))
\end{equation}
uniformly in the endpoint window. Consequently,
\begin{equation}
    \sum_{j=0}^{L_n} f_n(t-j)
    =
    f_n(t)(1+o(1))\sum_{j=0}^{L_n}e^{-j\delta_n}
    \sim
    f_n(t)\sum_{j=0}^{L_n}e^{-j\delta_n}.
\end{equation}
    Because $L_n\delta_n\to\infty$, the truncated geometric sum captures the full endpoint mass:
    \begin{equation}
        \sum_{j=0}^{L_n}e^{-j\delta_n}
        \sim
        \frac{1}{1-e^{-\delta_n}}
        \sim
        \frac{1}{\delta_n}.
    \end{equation}
    Combining this with \eqref{eq:local_uniform_Q}, we obtain
    \begin{equation}\label{eq:endpoint_contribution}
        \sum_{k=t-L_n}^{t}Q_{n,k}
        \sim
        \frac{f_n(t)}{\delta_n}.
    \end{equation}

    It remains to show that the contribution outside the endpoint window is negligible. We use the exact formula and the trivial bound
    \begin{equation}
        Q_{n,k}
        \le
        \frac{n}{(k+1)(k+2)}P_+(n,k).
    \end{equation}
    The same integral estimate gives the crude bound
    \begin{equation}\label{eq:crude_Q_bound}
        Q_{n,k}
        \le
        C\frac{n}{(k+1)(k+2)}
        \exp\Bigl(n\Psi(k/n)\Bigr)
    \end{equation}
    for all $k\le t$, where $C$ is independent of $k$ and $n$.

    First fix $\eta\in(0,1)$. For $k\le \eta t$, the exponent gap satisfies
    \begin{equation}
        n\bigl[\Psi(t/n)-\Psi(k/n)\bigr]
        \ge
        n\bigl[\Psi(t/n)-\Psi(\eta t/n)\bigr]
        =
        \frac{1-\eta^2}{2}\frac{t^2}{n}(1+o(1)).
    \end{equation}
    Since $t^2/n\to\infty$, the contribution of $k\le \eta t$ is exponentially smaller than the endpoint contribution:
    \begin{equation}\label{eq:low_part_negligible}
        \sum_{k=0}^{\lfloor \eta t\rfloor}Q_{n,k}
        =
        o\left(\frac{f_n(t)}{\delta_n}\right).
    \end{equation}

    It remains to control the intermediate range $\eta t<k<t-L_n$. Write $k=t-j$, so now $j\ge L_n$. By~\eqref{eq:crude_Q_bound} and the mean value theorem, uniformly for $\eta t\le k\le t$,
    \begin{equation}
        n\bigl[\Psi(t/n)-\Psi(k/n)\bigr]
        \ge
        c_\eta (t-k)\delta_n
    \end{equation}
    for some constant $c_\eta>0$, because $\Psi'(x)=\log(1+x)$ and $k$ is comparable to $t$ throughout this range. Hence
    \begin{equation}
        Q_{n,t-j}
        \le
        C_\eta f_n(t)e^{-c_\eta j\delta_n}
    \end{equation}
    for $L_n\le j\le (1-\eta)t$. Therefore
    \begin{equation}
        \sum_{\eta t<k<t-L_n}Q_{n,k}
        \le
        C_\eta f_n(t)\sum_{j\ge L_n}e^{-c_\eta j\delta_n}
        \le
        C_\eta f_n(t)\frac{e^{-c_\eta L_n\delta_n}}{\delta_n}.
    \end{equation}
    Since $L_n\delta_n\to\infty$, this gives
    \begin{equation}\label{eq:middle_part_negligible}
        \sum_{\eta t<k<t-L_n}Q_{n,k}
        =
        o\left(\frac{f_n(t)}{\delta_n}\right).
    \end{equation}
    Combining \eqref{eq:endpoint_contribution}, \eqref{eq:low_part_negligible}, and \eqref{eq:middle_part_negligible}, we conclude that
    \begin{equation}
        S_{n,t}
        \sim
        \frac{f_n(t)}{\delta_n}.
    \end{equation}

    Finally, by Theorem~\ref{thm:supercritical_regime},
    \begin{equation}
        Q_{n,t}\sim f_n(t),
    \end{equation}
    and therefore
    \begin{equation}
        S_{n,t}
        \sim
        \frac{Q_{n,t}}{\log(1+t/n)}.
    \end{equation}
    Since $\log(1+t/n)\sim t/n$ and $(t+1)(t+2)\sim t^2$, we also get
    \begin{equation}
        S_{n,t}
        \sim
        \frac{n^2}{t^3}
        \exp\Bigl(n\Psi(t/n)\Bigr).
    \end{equation}
\end{proof}

\begin{corollary}\label{cor:mesoscopic_regime_s}
    If $t/\sqrt{n}\to\infty$ and $t/n^{2/3}\to 0$, then
    \begin{equation}
        S_{n,t}\sim \frac{n^2}{t^3}\exp\Bigl(\frac{t^2}{2n}\Bigr).
    \end{equation}
\end{corollary}

\begin{proof}
    By Theorem~\ref{thm:supercritical_regime_s},
    \begin{equation}
        S_{n,t}\sim \frac{n^2}{t^3}\exp\Bigl(n\Psi(t/n)\Bigr).
    \end{equation}
    Since $t/n\to 0$, we may expand
    \begin{equation}
        \Psi(x)=\frac{x^2}{2}+O(x^3),
    \end{equation}
    and therefore
    \begin{equation}
        n\Psi(t/n)
        =
        \frac{t^2}{2n}
        +
        O\Bigl(\frac{t^3}{n^2}\Bigr).
    \end{equation}
    Under $t=o(n^{2/3})$, we have
    \begin{equation}
        \frac{t^3}{n^2}=o(1),
    \end{equation}
    and hence
    \begin{equation}
        n\Psi(t/n)=\frac{t^2}{2n}+o(1).
    \end{equation}
    Substituting this into the previous expression gives
    \begin{equation}
        S_{n,t}
        \sim
        \frac{n^2}{t^3}
        \exp\Bigl(\frac{t^2}{2n}\Bigr).
    \end{equation}
\end{proof}


\end{document}